\documentclass{article}

\usepackage{microtype}
\usepackage{graphicx}
\usepackage{subfigure}
\usepackage{booktabs} %
\usepackage{enumitem}

\usepackage{hyperref}

\usepackage[accepted]{icml2024}

\usepackage{amsmath}
\usepackage{amssymb}
\usepackage{mathtools}
\usepackage{amsthm}

\usepackage[capitalize,noabbrev]{cleveref}

\theoremstyle{plain}
\newtheorem{theorem}{Theorem}[section]
\newtheorem{proposition}[theorem]{Proposition}
\newtheorem{lemma}[theorem]{Lemma}

\theoremstyle{definition}
\newtheorem{definition}[theorem]{Definition}

\theoremstyle{remark}
\newtheorem{remark}[theorem]{Remark}

\DeclareMathOperator{\softmax}{\mathrm{softmax}}

\newcommand{\R}{\mathbb R}

\newcommand{\lip}{\mathrm{Lip}}
\newcommand{\norm}[1]{\left \lVert #1 \right \rVert}
\newcommand{\bbar}{\overline{\!{B}}}
\newcommand{\xcal}{\mathcal{X}}
\newcommand{\pcal}{\mathcal{P}}

\newcommand{\dd}{\mathrm{d}}
\newcommand{\var}{\mathrm{Var}}

\usepackage[textsize=tiny]{todonotes}

\icmltitlerunning{How Smooth Is Attention?}

\begin{document}

\twocolumn[
\icmltitle{How Smooth Is Attention?}

\icmlsetsymbol{equal}{*}

\begin{icmlauthorlist}
\icmlauthor{Valérie Castin}{ens}
\icmlauthor{Pierre Ablin}{apple}
\icmlauthor{Gabriel Peyré}{ens,cnrs}
\end{icmlauthorlist}

\icmlaffiliation{ens}{École Normale Supérieure PSL, Paris, France}
\icmlaffiliation{apple}{Apple, Paris, France}
\icmlaffiliation{cnrs}{CNRS}

\icmlcorrespondingauthor{Valérie Castin}{valerie.castin@orange.fr}

\icmlkeywords{Machine Learning, ICML}

\vskip 0.3in
]

\printAffiliationsAndNotice{}  %

\begin{abstract}
Self-attention and masked self-attention are at the heart of Transformers' outstanding success.
Still, our mathematical understanding of attention, in particular of its Lipschitz properties -- which are key when it comes to analyzing robustness and expressive power -- is incomplete.
We provide a detailed study of the Lipschitz constant of self-attention in several practical scenarios, discussing the impact of the sequence length $n$ and layer normalization on the local Lipschitz constant of both unmasked and masked self-attention.
In particular, we show that for inputs of length $n$ in any compact set, the Lipschitz constant of self-attention is bounded by $\sqrt{n}$ up to a constant factor and that this bound is tight for reasonable sequence lengths.
When the sequence length $n$ is too large for the previous bound to be tight, which we refer to as the mean-field regime, we provide an upper bound and a matching lower bound which are independent of $n$.
Our mean-field framework for masked self-attention is novel and of independent interest.
Our experiments on pretrained and randomly initialized BERT and GPT-2 support our theoretical findings.
\end{abstract}

\section{Introduction}

Introduced by \citet{vaswani2017attention}, Transformers and their multi-head attention mechanism \cite{bahdanau2014neural} have significantly changed the machine learning landscape in just a few years, by becoming state-of-art models on a wide variety of tasks, from natural language processing \cite{brown2020language, radford2019language, wolf2019huggingface} to computer vision \cite{dosovitskiy2020image, zhao2020point, zhai2022scaling, lee2019set}.
Despite this great empirical success, however, little is known from a theoretical perspective about the smoothness of Transformer architectures, particularly of self-attention, their main building block.
We tackle this problem by focusing on the Lipschitz properties of self-attention, especially on its local Lipschitz constant, which controls how fast the output can change with respect to the input in the neighborhood of each point of the domain.

Studying the Lipschitz continuity of neural networks is particularly relevant for various questions \cite{rosca2020case}.
It provides guarantees of adversarial robustness, in an attack-agnostic way \cite{szegedy2013intriguing, cisse2017parseval, tsuzuku2018lipschitz, anil2019sorting, weng2018evaluating}.
Identifying inputs with a high local Lipschitz constant and understanding which local perturbations trigger the biggest change in the output also allows for robustifying the network, for example using adversarial training \cite{goodfellow2014explaining, miyato2015distributional, moosavi2016deepfool, kurakin2016adversarial}.
The Lipschitz constant is also involved in generalization bounds \cite{sokolic2017robust, neyshabur2017exploring, bartlett2017spectrally, von2004distance}.
From a different perspective, Lipschitz-constrained neural networks can be used to estimate Wasserstein distances \cite{peyre2019computational}, enhance expressivity and improve the performance of deep models \cite{miyato2018spectral, dasoulas2021lipschitz}, and build invertible neural networks \cite{behrmann2019invertible, chen2019residual}.
Finally, bounding the Lipschitz constant of a neural network is an important step in the study of the associated neural ODE \cite{chen2018neural}, in particular of its well-posedness \cite{lu2019understanding, geshkovski2023emergence, geshkovski2023layernorm}.

Lipschitz continuity of feed-forward neural networks (FFNs) has been extensively studied and remains a hard problem.
Estimating numerically the Lipschitz constant of a FFN is indeed NP-hard \cite{virmaux2018lipschitz}, and theoretical bounds appear to be much larger than the actual Lipschitz constant \cite{virmaux2018lipschitz, fazlyab2019efficient, latorre2020lipschitz}.
The main theoretical difficulty here is to handle the composition of several layers more accurately than just bounding it by the product of spectral norms of weight matrices, as done by \citet{szegedy2013intriguing}.
Still, taken independently, each linear map or activation function has a known tight Lipschitz constant.
This is \emph{not} the case for Transformers: the self-attention map has an involved non-linear structure, which makes the estimation of its local Lipschitz constant challenging and brings into play completely different approaches than for FFNs \cite{kim2021lipschitz, vuckovic2021regularity}.

\subsection{Contributions}

We make the following contributions.
\begin{itemize}[nosep,wide, leftmargin=*]
    \item We derive a bound on the local Lipschitz constant of self-attention, which takes the form $C \sqrt{n}$
    with $n$ the sequence length of inputs and $C$ a constant factor that depends on the parameters of self-attention and on an upper bound $R$ on the magnitude of tokens (\autoref{thm:unnorm_general_bound}).
    We show that our bound is tight in $n$ for reasonable (i.e. not too large) sequence lengths (Proposition \ref{prop:lower_bound_sqrt}).
    Moreover, in most Transformer architectures, the magnitude $R$ only depends on the parameters of the network because of normalization layers (Subsection \ref{subsec:normalization}).
    \item We identify a \emph{large radius regime} that is easier to analyze theoretically, with $n$ fixed and $R$ very large.
    In this regime and except for a measure-zero set of pathological configurations, we show that the local Lipschitz constant of self-attention is bounded by $C \sqrt{n}$ with $C$ a constant that does not depend on $R$ anymore (\autoref{thm:unnorm_large_R_regime}).
    \item We also study the mean-field regime, where self-attention is modeled as a map on probability measures, which corresponds to the limit $n\to +\infty$.
    In this framework, we show that the upper bound obtained by \citet{geshkovski2023emergence}, which is of the form $C R^2 e^{CR^2},$
    cannot be significantly improved (Proposition \ref{prop:unnorm_lower_bound_mf}), by finding a $R$-indexed family of two-Dirac probability measures supported in the closed ball $B_R$ of center 0 and of radius $R$ whose local Lipschitz constant grows like $\frac{C'}{2} R^2 e^{C'R^2}$ with $C'\ge C/16$.
    \item We are the first to study the Lipschitz constant of \emph{masked} self-attention.
    We introduce a novel mean-field framework for masked self-attention, where the order of points in input measures is encoded in a supplementary coordinate, and show both in the general regime, the large radius regime and the mean-field regime that similar bounds hold for masked self-attention as for unmasked self-attention (\autoref{sec:masked}).
    \item We compute numerically the local Lipschitz constant of unmasked and masked self-attention in a BERT model and a GPT-2 model, where inputs are text extracts, and observe a growth rate of $n^{1/4}$ up to a constant factor, with $n$ the sequence length.
    Then, with the same networks, we build adversarial data in the input space of self-attention whose Lipschitz constant grows like $\sqrt{n}$, which evidentiates the tightness of our bounds (Section \ref{sec:experiments}).
\end{itemize}

\subsection{Related Work}

\paragraph{Robustness and local Lipschitz constant estimation.}
Neural networks are vulnerable to adversarial attacks \cite{szegedy2013intriguing}, and most of the methods proposed to measure and increase their robustness focus on specific attacks \cite{goodfellow2014explaining,papernot2016distillation}.
It turns out, however, that such methods can be defeated by well-chosen unseen attacks \cite{carlini2017towards}.
Measures of robustness that are agnostic to attack methods have therefore been proposed, often relying on the notion of Lipschitz constant of networks \cite{szegedy2013intriguing, leino2021globally, tsuzuku2018lipschitz}.
As robustness lower bounds that rely on the (global) Lipschitz constant tend to be too loose, tighter constraints have been proposed involving the local Lipschitz constant \cite{hein2017formal,weng2018evaluating}.
The problem of evaluating the local Lipschitz constant of deep networks is now at the heart of several recent articles \cite{tsuzuku2018lipschitz, leino2021globally}, in particular for Transformers \cite{kim2021lipschitz, Vuckovic2020AMT, geshkovski2023emergence, catellier2023robustness}.
From a more practical viewpoint, several Lipschitz-constrained variants of the Transformer architecture have been proposed, to increase robustness and reliability \cite{jia2023revisiting, gupta2023certvit, ye2023mitigating, qi2023lipsformer}.

\paragraph{Neural networks acting on measures.}
\citet{debie2019stochastic} and \citet{pevny2019approximation}
are the first to define neural networks whose inputs are probability measures, followed by several other articles \cite{Vuckovic2020AMT, zweig2021functional, sander2022sinkformers, geshkovski2023emergence}.
Modeling neural networks as maps on probability measures has multiple applications, such as studying Wasserstein regularity \cite{Vuckovic2020AMT, geshkovski2023emergence}, proving generalization bounds \cite{zweig2021functional} and doing a mean-field limit analysis of the dynamics of particles as they go through the network \cite{geshkovski2023emergence}.
The mean-field approach is particularly suited to the case of Encoder-only Transformers \cite{devlin2018bert}, as the self-attention map is permutation equivariant, i.e., ignores the order of vectors in its input.
This property can be leveraged to model any infinitely deep Encoder as a partial differential equation (PDE) on the space of measures \cite{sander2022sinkformers}, following the principle of neural ODEs \cite{chen2018neural}.
Analyzing this PDE then provides information about the dynamics of tokens as they go through the Transformer, showing for instance the emergence of clusters \cite{geshkovski2023emergence, geshkovski2023layernorm}.
In contrast, masked self-attention, which is crucial in Decoder-only \cite{liu2018generating} and Encoder-Decoder \cite{vaswani2017attention} architectures, is not permutation equivariant, so cannot be cast as naturally into a mean-field framework.

\paragraph{Regularity of self-attention and its variants.}
\citet{kim2021lipschitz} show that the self-attention map is not globally Lipschitz continuous by deriving a lower bound on its Lipschitz constant restricted to $B_R^n$.
Their lower bound grows quadratically with $R$.
To gain regularity, they define a new self-attention map, called L2 self-attention, which is globally Lipschitz continuous on the set of inputs of length $n$, for all $n\ge 1$.
\citet{dasoulas2021lipschitz} enforce the Lipschitz continuity of self-attention modules by normalizing the attention scores with a well-chosen normalization function.
\citet{geshkovski2023emergence} and \citet{Vuckovic2020AMT} prove a mean-field upper bound on the Lipschitz constant of self-attention on $B_R$, by viewing self-attention as a map acting on probability measures.
Their upper bound grows more than exponentially with $R$ so that the quadratic lower bound and the exponential upper bound put together provide a very loose estimation of the Lipschitz constant of self-attention on compact sets.
Finally, \citet{sander2022sinkformers} propose a modification of the attention kernel that builds on the Sinkhorn-Knopp algorithm, and provide empirical evidence of the better properties of this new choice of kernel with respect to the classical one.

\subsection{Notations}

The Euclidean norm on $\R^d$ is denoted $\lvert\, \cdot\, \rvert$.
For any vector $w\in \R^n$, we denote
$
    \softmax(w) \coloneqq \left ( \exp(w_i)/{\sum_{j=1}^d \exp(w_j)} \right )_{1\le i\le n}
$
the Softmax operator, and $\mathrm{diag}(w)$ the diagonal matrix such that $\mathrm{diag}(w)_{ii} = w_i$.
For any function $g\colon \mathcal{E} \to \mathcal{F}$ and any subset $\mathcal{X} \subset \mathcal{E}$, the restriction of $g$ to $\mathcal{X}$ is denoted $g_{\lvert \mathcal{X}}$.
The closed ball centered at 0 and of radius $R>0$ is denoted~$B_R$.
For $\varphi, \psi\colon \mathbb{R} \to \mathbb{R}$ and $a\in \mathbb{R}\cup \{+\infty\}$ we write $\varphi(x) \sim_{x \to a} \psi(x)$ if $\varphi(x)/ \psi(x)$ is well-defined for $x$ close enough to $a$, and $\varphi(x)/ \psi(x) \to_{x \to a} 1$.

\section{Standard and Mean-Field Self-Attention}

\subsection{Unmasked Self-Attention}

Unmasked self-attention, usually just called self-attention, is central in the architecture of Transformer's Encoders \cite{vaswani2017attention}, which are nowadays widely used for computer vision tasks \cite{dosovitskiy2020image}.
It maps sequences of $n$ vectors to sequences of $n$ vectors, for any integer $n$.

\begin{definition}[Single-head self-attention]
    \label{def:self_attention}
    Let $k, d \in \mathbb N$.
    Let $Q, K, V \in \mathbb R^{k\times d}$ be three matrices.
    For any integer $n\in \mathbb N$ and any vectors $x_1, \dots, x_n \in \mathbb R^d$, self-attention with parameters $(Q, K, V)$ maps the sequence $(x_1, \dots, x_n) \in (\mathbb R^d)^n$ to
    \begin{equation*}
        f(x_1, \dots, x_n) \coloneqq \Big ( V {\textstyle \sum_{j=1}^n} P_{ij} x_j \Big )_{1 \le i \le n} \in (\R^k)^n,
    \end{equation*}
    \begin{equation*}
        \text{with}\quad P_i \coloneqq \softmax \left ( (x_i^\top Q^\top K x_j / \sqrt{k})_{1\le j\le n} \right ).
    \end{equation*}
\end{definition}

To alleviate notations, we will denote $A \coloneqq K^\top Q / \sqrt{k} \in \R^{d\times d}$ in what follows.
In Encoders, several self-attention heads are usually combined to obtain multi-head self-attention.

\begin{definition}[Multi-head self-attention]
    \label{def:multi_head}
    Let $d\in \mathbb N$ and $H$ a divisor of $d$.
    For $1\le h\le H$, let $Q^{(h)}, K^{(h)}, V^{(h)}\in \R^{k\times d}$ with $k\coloneqq d / H$, and $W^{(h)}\in \R^{d\times k}$.
    Multihead self-attention with parameters $(Q^{(h)}, K^{(h)}, V^{(h)}, W^{(h)})_{1\le h \le H}$ maps any sequence $(x_1, \dots, x_n) \in (\mathbb R^d)^n$ to
    \begin{equation*}
        f^{MH}(x_1, \dots, x_n) \coloneqq \sum_{h=1}^H W^{(h)} f^{(h)}(x_1, \dots, x_n) \in (\R^d)^n,
    \end{equation*}
    where $f^{(h)}$ denotes single-head self-attention with parameters $(Q^{(h)}, K^{(h)}, V^{(h)})$.
\end{definition}

When $n$ is very large, it can be convenient to model self-attention as a map between probability measures \cite{sander2022sinkformers, geshkovski2023emergence}.
Indeed, the self-attention map $f$ is permutation equivariant, which means that for all permutations $\sigma$ of the set $\{1, \dots, n\}$ and all inputs $X = (x_1, \dots, x_n)\in (\R^d)^n$, it holds that $f(x_{\sigma(1)}, \dots, x_{\sigma(n)}) = (f(X)_{\sigma(1)}, \dots, f(X)_{\sigma(n)})$.
Informally, this means that self-attention is blind to the order of vectors so that it naturally induces a map between empirical measures, by replacing ordered sequences $X = (x_1, \dots, x_n)$ with their associated empirical measure $m(X) \coloneqq \frac{1}{n} \sum_{i=1}^n \delta_{x_i}$, which does not encode the order of points anymore.
To extend self-attention to more general probability measures, following \citet{sander2022sinkformers}, let us introduce the notion of pushforward.

\begin{definition}[\citet{santambrogio2015optimal}]
    For a probability measure $\mu$ on $\R^d$ and a measurable map $\varphi \colon \R^d \to \R^d$, the pushforward of $\mu$ by $\varphi$, denoted $\varphi_\sharp \mu$, is the probability measure given by
    $\left (\varphi_\sharp \mu \right )(B) \coloneqq \mu(\varphi^{-1}(B))$
    for any Borel set $B\subset \R^d$, where $\varphi^{-1}(B) \coloneqq \{ x\in \R^d : \varphi(x) \in B \}$.
\end{definition}

Intuitively, $\varphi_\sharp \mu$ is obtained by transporting each element of mass $\mu(\dd x)$ from $x$ to $\varphi(x)$.
We are now ready to define mean-field self-attention.

\begin{definition}[Mean-field self-attention]
    Let $Q, K, V \in \mathbb R^{k\times d}$, and denote $A\coloneqq K^\top Q / \sqrt{k}$.
    Mean-field self-attention with parameters $(A, V)$ is defined as
    \begin{equation*}
        F\colon \mu\in \mathcal{P}_c(\R^d) \mapsto \left (\Gamma_\mu\right )_\sharp \mu
    \end{equation*}
    $$\text{with} \quad \Gamma_\mu \colon x\in \R^d \mapsto \frac{\int \exp \left (x^\top A^\top y \right )Vy \, \dd \mu(y)}{\int \exp \left (x^\top A^\top y \right )\dd \mu(y)}.$$
\end{definition}

Mean-field self-attention $F$ generalizes discrete self-attention $f$ in the sense that for any input $X \in (\R^d)^n$, we have $F(m(X)) = m(f(X))$ (see Appendix \ref{appsubsec:mean_field_attention}).

\subsection{Masked Self-Attention}

In Decoder-only architectures, typically used for text generation \cite{liu2018generating, openai2023gpt4}, unmasked self-attention is replaced by masked self-attention.

\begin{definition}[Masked self-attention]
    \label{def:masked_self_attention}
    Let $Q, K, V \in \mathbb R^{k\times d}$, and $A \coloneqq K^\top Q / \sqrt{k}$.
    For any integer $n\in \mathbb N$ and any vectors $x_1, \dots, x_n \in \mathbb R^d$, residual masked self-attention with parameters $(A, V)$ maps the sequence $X = (x_1, \dots, x_n) \in (\mathbb R^d)^n$ to $(f^m(X)_1, \dots, f^m(X)_n) \in (\R^d)^n$, where 
    $$f^{m}(X)_i \coloneqq f(x_1, \dots, x_i)_i$$
    with $f$ unmasked self-attention (see \cref{def:self_attention}).
\end{definition}

Masked self-attention processes inputs sequentially, so it is not permutation equivariant and the map $f^m$ does not directly induce a map on empirical measures as for unmasked self-attention.
To overcome this limitation and still give meaning to masked self-attention when the sequence length goes to infinity, we introduce the following novel mean-field framework.
Instead of viewing inputs $(x_1, \dots, x_n) \in (\R^d)^n$ as empirical measures $\frac{1}{n} \sum_{i=1}^n \delta_{x_i}$, we add a coordinate $s_i\in[0, 1]$ to each $x_i$, to encode its position in the sequence.
We then define mean-field masked self-attention as a map between probability measures on the product space $[0,1] \times \R^d$.

\begin{definition}[Mean-field masked self-attention]
    For any probability measure $\bar \mu \in \pcal_c([0,1]\times \R^d)$, denote $\mu$ the second marginal of $\bar \mu$, i.e. $\mu(A) \coloneqq \int_{s=0}^1 \int_{x\in A} \dd \bar \mu(s, x)$ for all Borel sets $A \subset \R^d$.
    We define mean-field masked self-attention on $\pcal_c([0,1]\times \R^d)$ as
    $$F^m \colon \bar \mu \mapsto \left ( \Gamma_{\bar \mu} \right )_\sharp \bar \mu
    \quad \text{where}$$
    $$\Gamma_{\bar \mu}(s, x) \coloneqq \left (s, \frac{\int_{[0,1]\times \R^d} \exp(x^\top A^\top y) Vy \mathbf{1}_{\tau \le s} \dd \bar \mu(\tau, y)}{ \int_{[0,1]\times \R^d} \exp(x^\top A^\top y) \mathbf{1}_{\tau \le s}\dd \bar \mu(\tau, y)}\right ).$$
\end{definition}

This generalizes Definition \ref{def:masked_self_attention} in the following sense: given any increasing sequence $0\le s_1< \dots <s_n\le 1$, denoting $\mathrm{ord}$ the transformation
$\mathrm{ord} \colon X = (x_1, \dots, x_n) \in (\R^d)^n \mapsto \frac{1}{n}\sum_{i=1}^n \delta_{(s_i, x_i)} \in\pcal_c([0,1]\times \R^d),$
we have $F^m(\mathrm{ord}(X)) = \mathrm{ord}(f^m(X))$ for all $X\in (\R^d)^n$.

Beyond the mean-field analysis of Lipschitz continuity of masked self-attention, the map $F^m$ can be used in future work to model Decoders as partial differential equations on probability measures and study the dynamics of tokens as they go through the network, as \citet{geshkovski2023emergence, geshkovski2023layernorm} do for Encoders.\footnote{However, to do so one should rather set the first coordinate of $\Gamma_{\bar \mu}(s, x)$ to 0 instead of $s$ so that the residual map $(\mathrm{Id} + \Gamma_{\bar \mu})_\sharp \bar \mu$ preserves the first marginal.}

\subsection{Normalization}
\label{subsec:normalization}
Normalization is a key part of the Transformer architecture. 
The most common choice of normalization is LayerNorm~\citep{ba2016layer}, which has two learnable parameters $\gamma, \beta\in\R^d$.
It acts on each input of the sequence individually with the formula $\mathrm{norm}(x) = \gamma \odot \frac{x - \mathrm{mean}(x)}{\mathrm{std}(x)}  + \beta$, where $\mathrm{mean}(x) \coloneqq \tfrac{1}{d}\sum_{j=1}^d x_j$ and $\mathrm{std}(x) \coloneqq (\tfrac{1}{d}\sum_{j=1}^d(x_j - \mathrm{mean}(x))^2)^{1/2}$ are two scalars that depend on $x$.
Each vector of the output of LayerNorm is on an ellipsis $\mathcal{S}$ of center $\beta$ and of covariance $\mathrm{diag}(\gamma)^2d$.
Another popular and simpler normalization is RMSNorm~\citep{zhang2019root}, which has one learnable parameter $\gamma\in\R^d$ and acts on each input of the sequence individually with the formula $\mathrm{norm}(x) = \gamma \odot \frac{x}{\lvert x\rvert}\sqrt{d}$.
RMSNorm is used in recent large language~\citep{jiang2023mistral,touvron2023llama} and vision~\citep{dehghani2023scaling} models.
Each vector of the output of RMSNorm is on an ellipsis $\mathcal{S}$ centered at $0$ and of covariance $\mathrm{diag}(\gamma)^2d$.
There are two ways to place the normalization layers in the transformer. 
The original transformer uses \emph{post}-normalization: normalization is applied after each residual connection.
Letting $X = (x_1, \dots, x_n)$, the output of residual attention is therefore $\mathrm{norm}(X + f(X))$.
However, \emph{pre}-normalization~\citep{xiong2020layernorm}, where normalization is applied before the attention layer: $X + f(\mathrm{norm}(X))$, is now more widespread. 
Although the two formulations are not equivalent, the input of self-attention $f$ is normalized in both cases --- by definition for pre-normalization, and because the previous layer was normalized for post-normalization.
Hence, in practice, the input of self-attention is not any sequence in $(\R^d)^n$, but a sequence in $B_R^n$ for $R$ depending only on the parameters of $\mathrm{norm}$.

It is also worth noticing that for \mbox{RMSNorm}, the parameter $\gamma$ can be absorbed in the parameters $\theta \coloneqq (Q, K, V)$ of self-attention:
$$f_\theta \circ \mathrm{norm}(x_1, \dots, x_n) = f_{\theta'}\left (\frac{x_1}{\lvert x_1\rvert}, \dots, \frac{x_n}{\lvert x_n\rvert}\right )$$
with $\theta' \coloneqq (Q \mathrm{diag}(\lambda), K\mathrm{diag}(\lambda), V\mathrm{diag}(\lambda))$.
In other words, \mbox{RMSNorm} followed by self-attention is equivalent to a projection on the unit sphere followed by self-attention with different parameters.
This provides a simple way to bound the Lipschitz constant of the composition $f_\theta \circ \mathrm{norm}$, by directly applying our bounds on the Lipschitz constant of $f_{\theta'}$ for $R=1$.

\section{Tight Bounds on the Lipschitz Constant of Self-Attention}
\label{sec:bounds}

\subsection{Lipschitz Constants and Self-Attention}
\label{subsec:lipschitz_constants}

Lipschitz constants provide a natural way of controlling the regularity of a function.
Their definition depends on the structure that is chosen for the input and output spaces.

\paragraph{Euclidean framework.} 
Let $d,n\in \mathbb N$ and $f\colon (\R^d)^n \to (\R^d)^n$.
We equip the input and output spaces of $f$ with the Frobenius norm
\begin{equation*}
    \norm{X}_F \coloneqq (\textstyle\sum_{i=1}^n \lvert x_i \rvert^2 )^{1/2}
\end{equation*}
for any sequence of vectors $X=(x_1, \dots, x_n)$, and assume that $f$ is differentiable.
The local Lipschitz constant of $f$ at an input $X=(x_1, \dots, x_n)$ is defined as
\begin{equation*}
    \lip_X f \coloneqq \norm{D_X f}_2,
\end{equation*}
where $D_X f$ is the differential of the function $f$, and $\norm{\cdot}_2$ denotes the operator norm induced by $\norm{\cdot}_F$.
We can also define, for any subset $\xcal \subset (\R^d)^n$, the Lipschitz constant of $f$ on $\xcal$, as
\begin{equation*}
    \lip \left ( f_{\lvert \xcal} \right ) \coloneqq \sup_{\underset{X\neq Y}{X, Y \in \xcal}} \frac{\norm{f(X) - f(Y)}_F}{\norm{X - Y}_F}.
\end{equation*}

The local Lipschitz constant tells us how fast the output of $f$ can vary locally, while the Lipschitz constant on $f$ controls how fast the output of $f$ can vary on the whole set $\xcal$.
We have the following connection between the two.

\begin{lemma}[\citet{federer2014geometric}]
    Let $\xcal$ be an open and connected subset of $(\R^d)^n$.
    Then
    \begin{equation*}
        \lip (f_{\lvert \xcal}) = \sup_{X \in \xcal} \norm{D_X f}_2.
    \end{equation*}
\end{lemma}

\paragraph{Mean-field framework.}
Let $d\in \mathbb N$ and denote $\pcal_c(\R^d)$ the set of compactly supported probability measures on $\R^d$.
We equip this set with the 2-Wasserstein distance, defined as
\begin{equation*}
    W_2(\mu, \nu) \coloneqq \Big (\inf_{\pi \in \Pi (\mu, \nu)} \int \lvert x-y\rvert^2 \dd \pi(x,y)\Big)^{1/2}
\end{equation*}
for $\mu, \nu \in \pcal_c(\R^d)$, where $\Pi (\mu, \nu)$ is the set of couplings between $\mu$ and $\nu$, i.e. of probability measures $\pi \in \mathcal{P}(\R^d \times \R^d)$ such that $\int \pi(\cdot, y) \dd y = \mu$ and $\int \pi(x, \cdot) \dd x = \nu$ (see for example \citet{santambrogio2015optimal} for more details on the subject).
Consider a map $F\colon \pcal_c(\R^d) \to \pcal_c(\R^d)$.
For any subset $\xcal \subset \pcal_c(\R^d)$, the Lipschitz constant of $F$ on $\xcal$ is defined as
\begin{equation*}
    \lip(F_{\xcal}) \coloneqq \sup_{\underset{\mu \neq \nu}{\mu, \nu \in \xcal}} \frac{W_2(F(\mu), F(\nu))}{W_2(\mu, \nu)}.
\end{equation*}
A notion of local Lipschitz constant can also be defined in the mean-field framework.
We defer it to Appendix \ref{appsubsec:general_local_lip_cst} as it only appears in some of our proofs.

Measuring the regularity of self-attention in Wasserstein distance is the natural generalization to the mean-field case of the Euclidean regularity in the case of a finite sequence length.  
Indeed, when $f$ is self-attention and $F$ is its mean-field generalization, we can connect the two frameworks as follows (see Appendix \ref{appsubsec:link_euclidian_meanfield}).

\begin{lemma}
    \label{lem:link_euclidian_meanfield}
    Let $R > 0$.
    We have
    $$\lip^{\norm{\cdot}_F} (f_{\lvert B_R^n}) \le \lip^{W_2} (F_{\lvert \pcal(B_R)}).$$
\end{lemma}

\subsection{Lipschitz Bounds for Self-Attention in Different Regimes}

\paragraph{General upper bound.}
\citet{kim2021lipschitz} show that self-attention is not globally Lipschitz continuous.
Let us therefore restrict to sequences $(x_1, \dots, x_n)\in B_R^n$, where $B_R\subset \mathbb R^d$ is the closed ball centered at 0 and of radius $R$.
We have the following general bound (see Appendix \ref{appsubsec:unnorm_general_bound}).

\begin{theorem}
    \label{thm:unnorm_general_bound}
    Let $Q, K, V \in \R^{k\times d}$ and $A \coloneqq K^\top Q / \sqrt{k}$.
    Let $R > 0$ and $n \in \mathbb N$.
    Unmasked self-attention $f$ with parameters $(A, V)$ is Lipschitz continuous on $B_R^n$, with
    \begin{equation*}
        \lip \Big ( f _{\lvert B_R^n} \Big ) \le \sqrt{3} \norm{V}_2 \left ( \norm{A}_2^2 R^4 (4n + 1) + n \right )^{1/2}\!\!.
    \end{equation*}
\end{theorem}

Theorem \ref{thm:unnorm_general_bound} shows that when tokens are restricted to the compact set $B_R$, the Lipschitz constant of self-attention grows at most like $\sqrt{n}$ with the sequence length $n$, up to a constant factor.
On the other side, the growth rate $\sqrt{n}$ in Theorem \ref{thm:unnorm_general_bound} is tight as long as $n$ is not too large, a statement made rigorous by the following result (see Appendix \ref{appsubsec:lower_bound_sqrt}).

\begin{proposition}
    \label{prop:lower_bound_sqrt}
    Let $Q, K \in \R^{k\times d}$ and $V\coloneqq I_d$.
    Let $A \coloneqq K^\top Q / \sqrt{k}$.
    Denote $f$ unmasked self-attention with parameters $(A, V)$.
    Let $\gamma_1 \ge \dots \ge \gamma_\delta$ be the real eigenvalues of $A$.
    Then, for any $R > 0$, and denoting $\gamma \coloneqq \max(- \gamma_\delta, \gamma_1 / 8)$, it holds
    $$\lip(f_{\lvert B_R^n}) \ge \frac{1}{1 + (n-1)e^{- 2 R^2 \gamma}} \sqrt{n - 1}.$$
\end{proposition}

Proposition \ref{prop:lower_bound_sqrt} shows that for any radius $R > 0$, the sequence of Lipschitz constants $ (\lip f_{\lvert B_R^n} )_{n\in \mathbb N}$ grows faster than $\sqrt{n}$ up to a constant factor in a certain range of sequence lengths $n$.
For example, if $n\le 1 + e^{2R^2 \gamma}$, then 
$$\lip(f_{\lvert B_R^n}) \ge \frac{\sqrt{n - 1}}{2},$$
and we check that for real data and pretrained GPT-2 and BERT, the factor $R^2 \gamma$ is of the order of $10^2$ to $10^3$ (see Appendix \ref{appsubsec:scaling_factor}) so that the inequality $n\le 1 + e^{2R^2 \gamma}$ is always satisfied in practice.
Note that Proposition \ref{prop:lower_bound_sqrt} gives a worst-case lower bound: there are inputs with a large sequence length and a small local Lipschitz constant, such as $X\coloneqq (x, \dots, x) \in (\R^d)^n$ for any $x\in\R^d$ and $n$, which satisfies $\norm{D_X f}_2 = 1$.

\paragraph{Mean-field regime.}
What happens when the sequence length $n$ is extremely large?
As explained above, the bound provided by Theorem \ref{thm:unnorm_general_bound} becomes too loose -- it even goes to $+\infty$ when $n \to +\infty$ with fixed $R$ and $r$.
For very large sequence lengths, this bound can be refined by leveraging the mean-field framework, as follows.

\begin{theorem}
    \label{thm:mean_field_unnorm}
    Let $R>0$ and $A, V \in \R^{k\times d}$.
    The mean-field self-attention map $F$ with parameters $(A, V)$ is $W_2$-Lipschitz continuous on the set $\mathcal{P}(B_R)$, and its Lipschitz constant is bounded by
    $$\lip^{W_2}(F_{\lvert \mathcal{P}(B_R})) \le \norm{V}_2 (1 + 3\norm{A}_2R^2) e^{2\norm{A}_2 R^2}.$$
\end{theorem}

Theorem \ref{thm:mean_field_unnorm} follows from computations made by \citet{geshkovski2023emergence}.
We state it to draw a full picture of the regularity of self-attention on compact sets.
Let us highlight the following connection between Theorem \ref{thm:mean_field_unnorm} and Theorem \ref{thm:unnorm_general_bound}.
For any radius $R >0$, sequence length $n\in \mathbb N$ and input $X\in B_R^n$, it holds, according to Lemma \ref{lem:link_euclidian_meanfield} and Theorem \ref{thm:mean_field_unnorm}:
\begin{equation}
    \label{eq:mean_field_bound}
    \norm{D_Xf}_2 \le \norm{V}_2 (1 + 3\norm{A}_2R^2) e^{2\norm{A}_2 R^2}.
\end{equation}
On the other hand, Theorem \ref{thm:unnorm_general_bound} tells us that
$$\norm{D_Xf}_2 \le \sqrt{3} \norm{V}_2 \left ( \norm{A}_2^2 R^4 (4n + 1) + n \right )^{1/2}.$$
When $n$ is relatively small, Theorem \ref{thm:unnorm_general_bound} is more relevant than Equation (\ref{eq:mean_field_bound}), and vice versa for $n$ very large.
In the following Proposition, we identify an edge regime where both bounds have a similar growth in $R$, which corresponds to $n\sim_{R\to +\infty}e^{cR^2}$ for some constant factor $c > 0$.
In this edge regime, the bound of Theorem \ref{thm:unnorm_general_bound} appears to be tight both in $n$ and $R$, and the bound of Theorem \ref{thm:mean_field_unnorm} is almost tight in $R$ up to a constant factor in the exponential.

\begin{proposition}
    \label{prop:unnorm_lower_bound_mf}
    Let $R>0$. 
    Assume that $k = d$ and $V = I_d$, and denote $\gamma_1\ge \dots \ge \gamma_\delta$ the real eigenvalues of $A$.
    Let $\gamma \coloneqq \max(- \gamma_\delta, \gamma_1 / 8)$.
    Then, if $n\sim_{R\to +\infty} \exp(2 \gamma R^2)$, there exists a function $\theta \colon [0, +\infty) \to [0, +\infty)$ such that $\theta(R) \to_{R\to +\infty} 1$ and:
        $$\lip(f_{\lvert B_R^n}) \ge  \theta (R) \frac{\gamma}{2} R^2  e^{\gamma R^2}.$$
\end{proposition}

One sees that the dependency in $R$ of the lower bound in Proposition \ref{prop:unnorm_lower_bound_mf} is catastrophic.
Fortunately, in practical cases, $n$ is significantly smaller than $e^{2\gamma R^2}$, and the mean-field regime bound is over-pessimistic: one should rather use Theorem \ref{thm:unnorm_general_bound}.
It is also interesting to note that configurations that lead to the explosion of the right-hand side in Proposition \ref{prop:unnorm_lower_bound_mf} are made of two extremely unbalanced clusters, one with 1 vector, and the other with the other $n-1$ vectors of the sequence (see Appendix \ref{appsubsec:unnorm_lower_bound_mf}).

\paragraph{Large radius regime.}
Let us now analyze a third regime: the large radius regime, where $R$ goes to infinity with a fixed $n$.
This complements the mean-field analysis, where $R$ is fixed and $n$ goes to infinity.
Let $n \in \mathbb N$ be a fixed sequence length.
We show, drawing inspiration from \citet{kim2021lipschitz}, that there exist configurations with $n$ particles supported in $B_R$ whose local Lipschitz constant grows like $R^2$ up to constant factors (see Appendix \ref{appsubsec:quadratic_growth}).
However, if we rule out a measure-zero set of pathological configurations,
we get in the large radius regime that the Lipschitz constant grows at most like $\sqrt{n}$ up to a \emph{constant} factor.

\begin{theorem}
    \label{thm:unnorm_large_R_regime}
    Let $A \in \R^{d\times d}$ and $V\in \R^{k\times d}$ two non-zero matrices.
    Denote $\mathcal{E}_A \subset \bbar(0,1)^n$ the set of sequences $(x_1, \dots, x_n)$ such that for all $i \in \{1, \dots, n\}$, the maximum $\max_{1\le j \le n} x_i^\top A^\top x_j$ is reached at a unique index $j$, denoted $m_i$.
    The complement of $\mathcal{E}_A$ in $\bbar(0,1)^n$ has zero Lebesgue measure.
    Moreover, for any $X \in \mathcal{E}_A$, there exists a function $\theta\colon [0, +\infty) \to [0, +\infty)$ such that $\theta(R)$ converges exponentially fast to 1 when $R\to +\infty$ and
    $$\norm{D_{RX}f}_2 \le \theta(R) \norm{V}_2 \sqrt{n}.$$
\end{theorem}

The proof of this result is interesting, as it provides a better understanding of the Jacobian of self-attention in this limiting regime.

\begin{proof}
    The sequences $(x_1,\dots, x_n) \in (\R^d)^n$ that are not in $\mathcal{E}_A$ are such that either there exists an index $i \in \{ 1, \dots, n\}$ for which $A x_i = 0$, or the $x_i$ are not distinct.
    Both cases are measure-zero situations, as $A\neq 0$.
    Now let $X \in \mathcal{E}_A$.
    For all perturbations $\epsilon \coloneqq (\epsilon_1,\dots, \epsilon_n) \in (\R^d)^n$ and all $i \in \{1, \dots, n\}$, we have (see Appendix \ref{appsubsec:jacobian_formula})
    \begin{multline*}
        (D_{RX}f)(\epsilon)_i = VR^2 \sum_{j=1}^n P_{ij} (x_j - \sum_{k=1}^n P_{ik}x_k)x_i^\top A^\top \epsilon_j \\ 
        + V\sum_{j=1}^n P_{ij} \epsilon_j + VR^2 \sum_{j=1}^n P_{ij}(x_j - \sum_{k=1}^n P_{ik}x_k) x_j^\top A \epsilon_i,
    \end{multline*}
    with $P_{ij} \coloneqq e^{R^2 x_i^\top A^\top x_j} / \sum_{k=1}^n e^{R^2 x_i^\top A^\top x_k}$.
    By definition of $m_i$, we have $P_{ij} \to_{R\to +\infty} \mathbf{1}_{j = m_i}$, and the convergence is exponentially fast, so $R^2 P_{ij}$ has the same limit as $P_{ij}$.
    As a consequence, in the large radius regime, the Jacobian of self-attention has a much simpler form:
    $(D_{RX}f)(\epsilon)_i \to_{R\to +\infty} V\sum_{j=1}^n P_{ij} \epsilon_j,$
    and the operator norm of the limit is bounded by $\norm{V}_2 \sqrt{n}$.
    Moreover, when $V = I_d$ for example, this bound is reached up to a constant factor if there exists an index $j$ such that $j = m_i$ for a constant fraction of the indices $i$, i.e. if a token grasps the attention of a constant fraction of all tokens.
\end{proof}

In practice, the large radius regime on general configurations (i.e. that belong to the set $\mathcal{E}_A$) provides an oversimplified model for self-attention.
Indeed, in this regime, attention matrices are 1-sparse row-wise, i.e. have exactly one non-zero coefficient -- equal to 1, on each row $i\in \{1,\dots,n\}$, which corresponds to the index $m_i$.
If we look at real data, attention matrices indeed tend to have sparse rows, but with more than one non-zero coefficient on each row \cite{vaswani2017attention, vyas2020fast, likhosherstov2021expressive} -- which is expected, otherwise the representation given by self-attention would be too rough.
Still, Theorem \ref{thm:unnorm_large_R_regime} gives some nice intuition about the growth rate of $\sqrt{n}$ obtained in Theorem \ref{thm:unnorm_general_bound} and observed in practice (see Figure \ref{fig:real_data} in the experiments).

\paragraph{Multi-head self-attention.}
Bounding the Lipschitz constant of single-head self-attention provides the following bound on the Lipschitz constant of multi-head self-attention.

\begin{lemma}[\citet{kim2021lipschitz}]
    \label{lem:multi_head}
    Let $R>0$.
    With the notations of Definition \ref{def:multi_head}, it holds
    $$\lip(f^{MH}_{\lvert B_R^n}) \le \sum_{h=1}^H \|W^{(h)}\|_2 \, \lip(f^{(h)}_{\lvert B_R^n}).$$
\end{lemma}

In the whole paper, we focus on single-head self-attention and avoid tackling the possibly tedious dependencies between the matrices $W^{(1)}V^{(1)}, \dots, W^{(H)}V^{(H)}$.
Deriving a finer estimation of the Lipschitz constant of multi-head attention than what Lemma \ref{lem:multi_head} gives us is left for future work.

\section{Tight Bounds on the Lipschitz Constant of Masked Self-Attention}
\label{sec:masked}

\subsection{Measuring Distances With Conditional Optimal Transport}

To study the Lipschitz properties of masked self-attention as defined in Definition \ref{def:masked_self_attention}, the Euclidean framework introduced in \autoref{subsec:lipschitz_constants} still applies.
In contrast, the Wasserstein framework used for mean-field unmasked self-attention is not suited to measuring the regularity of mean-field masked self-attention.
Indeed, in the standard case, the distance between the outputs $f^m(X)$ and $f^m(Y)$ for two inputs $X,Y\in (\R^d)^n$ is measured by $(\sum_{i=1}^n \lvert f(x_1, \dots, x_i)_i - f(y_1, \dots, y_i)_i\rvert^2)^{1/2}$ so that $i$-th coordinates are compared to each other, separately for each index $i$.
On the contrary, when looking at $W_2(F^m(\bar \mu), F^m(\bar \nu))$, the optimal transport plan may transport mass from a point $(s,x)$ to a point $(s', y)$ with $s\neq s'$, which contradicts the sequential nature of masked self-attention.
It is therefore natural to introduce the following distance on $\pcal_c([0,1]\times \R^d)$, which comes from the theory of conditional optimal transport \cite{hosseini2023conditional}.

\begin{definition}
    Let $\bar \mu$ and $\bar \nu$ be two probability measures in $\pcal_c([0,1]\times \R^d)$, and $p \ge 1$.
    If $\bar \mu$ and $\bar \nu$ have the same marginal with respect to $s$, i.e.
    $$\int_{s_1}^{s_2} \int_{\R^d} \dd \bar \mu(s, x) = \int_{s_1}^{s_2} \int_{\R^d} \dd \bar \nu(s, x)$$
    for all $0\le s_1 < s_2 \le 1$, denote $\theta$ this marginal distribution, and write with the disintegration theorem \cite{bogachev2007measure}
    $\dd \bar \mu (\tau, x) \eqqcolon \dd \theta(\tau) \dd \mu^\tau (x)$ and $\dd \bar \nu (\tau, x) \eqqcolon \dd \theta(\tau) \dd \nu^\tau (x).$
    The measures $\mu^\tau$ and $\nu^\tau$ can be seen intuitively as the restriction of $\mu$ and $\nu$ to the mass located at position $\tau$, rescaled to obtain probability measures.
    We then measure the distance between $\bar \mu$ and $\bar \nu$ with
    $$d_p(\bar \mu, \bar \nu) \coloneqq \left (\int_0^1 W_p(\mu^\tau, \nu^\tau)^p \dd \theta(\tau)\right )^{1/p}.$$
    If $\bar \mu$ and $\bar \nu$ do not have the same first marginal, we set
    $d_p(\bar \mu, \bar \nu) \coloneqq +\infty.$
\end{definition}

Considering $d_p(\bar \mu, \bar \nu)$ amounts to minimizing the transport cost between $\bar \mu$ and $\bar \nu$ under the constraint that points must keep the first coordinate constant along their trajectory.
Equivalently, allowed transport plans must be the identity on the first marginal.

As for unmasked self-attention, we have the following connection between the Euclidean framework and the mean-field framework for residual masked self-attention.

\begin{lemma}
    Let $R > 0$.
    We have
    $$\lip^{\norm{\cdot}_F} (f^m_{\lvert B_R^n}) \le \lip^{d_2} (F^m_{\lvert \pcal([0,1]\times B_R)}).$$
\end{lemma}

We do not detail the proof, which follows the same steps as for Lemma \ref{lem:link_euclidian_meanfield}.

\subsection{Lipschitz Bounds for Masked Self-Attention in Different Regimes}

\paragraph{General upper bound.}
We show in Appendix \ref{appsubsec:masked_general_bound} that the bound provided by \autoref{thm:unnorm_general_bound} still holds for masked self-attention.

\begin{theorem}
    \label{thm:masked_general_bound}
    Let $Q, K, V \in \R^{k\times d}$ and $A \coloneqq K^\top Q / \sqrt{k}$.
    Let $R > 0$ and $n \in \mathbb N$.
    Masked self-attention $f^m$ with parameters $(A, V)$ is Lipschitz continuous on the set $B_R^n$, and
    \begin{equation*}
        \lip \left ( f^m_{\lvert B_R^n} \right ) \le \sqrt{3} \norm{V}_2 \left ( \norm{A}_2^2 R^4 (n + 1) + n \right )^{1/2}.
    \end{equation*}
\end{theorem}

\paragraph{Mean-field regime.}
Let us now bound from above the $d_p$ Lipschitz constant of mean-field masked self-attention.

\begin{theorem}
    \label{thm:masked_mean_field}
    Let $R > 0$ and $p\ge 1$.
    The mean-field masked self-attention map $F^m$ is Lipschitz continuous on the space of measures supported in $[0,1] \times B_R$, with a Lipschitz constant upper-bounded by
    $$\norm{V}_2(1 + 3\norm{A}_2 R^2) e^{2\norm{A}_2 R^2}.$$
\end{theorem}

Note that in \autoref{thm:masked_mean_field}, we consider that two measures $\bar \mu$ and $\bar \nu$ with different first marginals induce an infinite Lipschitz ratio $d_p(F^m(\bar \mu), F^m(\bar \nu)) / d_p(\bar \mu, \bar \nu)$.
The proof can be found in Appendix \ref{appsubsec:masked_mean_field_bound}.

\paragraph{Large radius regime.}
We have a similar result as for unmasked attention in the large radius regime.
Except for a measure-zero set of pathological configurations, when $R$ is sufficiently large, the local Lipschitz constant of $f^m$ at the input $RX$ does not depend on $R$ anymore and grows at most like $\sqrt{n}$ up to a constant factor.

\begin{theorem}
    \label{thm:large_R_masked}
    Let $A \in \R^{d\times d}$ and $V\in \R^{k\times d}$ two non-zero matrices, and denote $f^m$ the masked self-attention map with parameters $(A, V)$.
    Denote $\mathcal{E}^m_A \subset \bbar(0,1)^n$ the set of sequences $(x_1, \dots, x_n)$ such that for all $i \in \{1, \dots, n\}$, the maximum $\max_{1\le j \le i} x_i^\top A^\top x_j$ is reached at a unique index $j$, denoted $m_i$.
    The complement of $\mathcal{E}^m_A$ in $\bbar(0,1)^n$ has zero Lebesgue measure.
    Moreover, for any $X \in \mathcal{E}_A$, there exists a function $\theta\colon [0, +\infty) \to [0, +\infty)$ such that $\theta(R)$ converges exponentially fast to 1 when $R\to +\infty$ and
    $$\norm{D_{RX}f^m}_2 \le \theta(R) \norm{V}_2 \sqrt{n}.$$
\end{theorem}

\section{Experiments}
\label{sec:experiments}

\begin{figure*}
\centering
\includegraphics[width=0.9\textwidth]{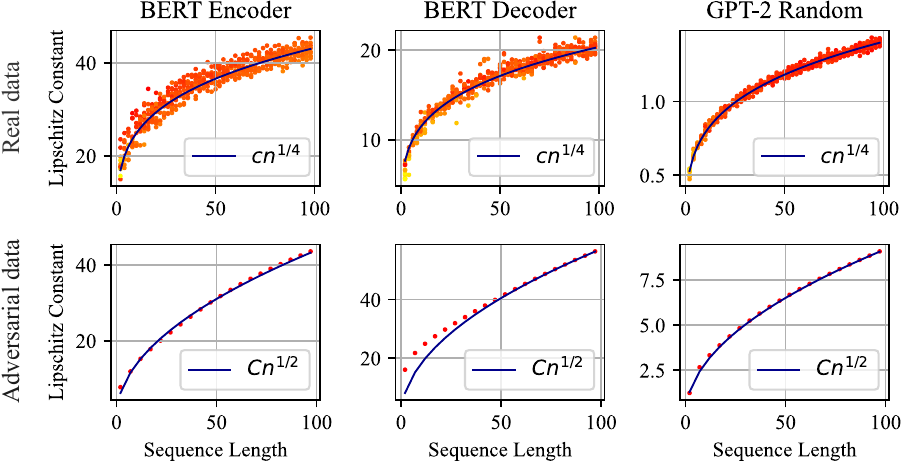}
\caption{Scatter plots of the local Lipschitz constant of self-attention (column 1) and masked self-attention (columns 2 and 3) on text data (upper row) and adversarial data (lower row) as a function of the sequence length $n$.
In the upper row, the color encodes the mean radius of inputs $X = (x_1, \dots, x_n)$, defined as $R\coloneqq \sqrt{1/n\sum_{i=1}^n\lvert x_i\rvert^2}$.
Lighter points have a smaller mean radius.
The first two columns correspond to two different pretrained BERT models: an Encoder-only and a Decoder-only, on the same dataset Alice in Wonderland, respectively for attention layers 0 and 6.
The third column is obtained with the masked self-attention layer 6 of GPT-2 randomly initialized, on the dataset AG\_NEWS.
We see that the Lipschitz constant of self-attention on real data grows approximately like $n^{1/4}$ with the sequence length $n$ and that the growth rate is $\sqrt{n}$ for adversarial data, which shows the tightness of Theorems \ref{thm:unnorm_general_bound}, \ref{thm:unnorm_large_R_regime}, \ref{thm:masked_general_bound} and \ref{thm:large_R_masked}.}
\label{fig:real_data}
\end{figure*}

The bound stated in \autoref{thm:unnorm_general_bound} corresponds to a worst-case scenario.
In practice, does it reflect the evolution of the Lipschitz constant of a self-attention layer of a Transformer on real data?
We perform numerical experiments on BERT \cite{devlin2018bert}, using the pretrained Huggingface model 'bert-base-uncased' first as an Encoder and then in its Decoder version, and on GPT-2 \cite{radford2019language} both pretrained and randomly initialized.
Both models have 12 multi-head attention layers, and 12 attention heads per layer, with an embedding dimension $d=768$.
We perform two different experiments, first with real data, and then with synthetic adversarial data.

\subsection{Experiments With Real Data}

We take our data from two test datasets, Alice in Wonderland from the NLTK corpus Gutenberg \cite{bird2009natural}, and AG\_NEWS from the PyTorch package torchtext \cite{Zhang2015CharacterlevelCN}.
The aim is, for various multi-head self-attention layers $f^{\text{model}}$ of BERT and GPT-2, and for a batch of inputs of varying length taken from the two datasets mentioned above, to get a scatter plot of the local Lipschitz constant of $f^{\text{model}}$ at each input $(x_1, \dots, x_n)$ as a function of the sequence length $n$.

\paragraph{Construction of the datasets.}
Given some raw text from Alice in Wonderland or AG\_NEWS, we first tokenize it and then split the resulting sequence of tokens into subsequences with a fixed sequence length.
For each even integer $n$ in $\{2, \dots, 100\}$, we build 10 sequences with $n$ tokens, so that none of the constructed sequences $(s_1, \dots, s_n)$ overlap.
Then, for each input sequence $(s_1, \dots, s_n)$, we do a forward pass of the model, and get with a forward hook the intermediate activations just before the attention layer of interest $f^{\text{model}}$.
This gives us a batch of sequences $(x_1,\dots, x_n) \in (\R^d)^n$ that are fed to $f^{\text{model}}$ when $(s_1,\dots, s_n)$ goes through the model.
Note that except for the inputs of the first attention layer, all the $x_i$ are the result of normalization with LayerNorm, and therefore belong to an ellipsis, which depends on the parameters of LayerNorm.

\paragraph{Computation of the local Lipschitz constants.}
The local Lipschitz constant of $f^{\text{model}}$ at an input sequence $X = (x_1, \dots, x_n)$ is equal to $\norm{D_X f^{\text{model}}}_2$.
As $D_X f^{\text{model}}$, which we denote $J_X$ to alleviate notations, is of shape $nd \times nd$ with $d=768$, we do not compute it explicitly but use a power method on the matrix $J_X^\top J_X$ by alternating Jacobian-vector products and vector-Jacobian products (see Appendix \ref{appsubsec:power_iteration}).
The power method converges to the largest eigenvalue of $J_X^\top J_X$, which is equal to $\norm{J_X}_2^2$.

\subsection{Experiments With Adversarial Data}

\paragraph{Adversarial data.}
To check numerically the tightness in $n$ of the bound provided by Theorem \ref{thm:unnorm_general_bound}, we build adversarial data in the input space of each self-attention layer $f^{\text{model}}$.
More precisely, for each sequence length $n$, and for a radius $R>0$ to be discussed later, we look for an input $X\in B_R^n$ where $f^{\text{model}}$ has a large (ideally maximal) local Lipschitz constant.
Unfortunately, performing a gradient ascent on the local Lipschitz constant $f^{\text{model}}$ gives poor results, as the optimization landscape is highly non-convex.
We, therefore, build $X$ as follows, using the heuristics provided by Proposition \ref{prop:unnorm_lower_bound_mf}.
For $h\in \{1, \dots, 12\}$, denote $A^{(h)} \coloneqq {K^{(h)}}^\top Q^{(h)} / \sqrt{k}$, with the notations of Definition \ref{def:multi_head} applied to the multi-head self-attention layer $f^{\text{model}}$.
Choose any head $h$, and denote $\gamma_1$ (resp. $\gamma_\delta$) the largest (resp. the smallest) real eigenvalue of $A$, and $u_1$ (resp. $u_\delta$) an associated unit eigenvector.
If $\gamma_1 \ge - 8 \gamma_\delta$, define $X = (x_1, \dots, x_n)$ with $x_1 \coloneqq R u_1$ and $x_2 = \dots = x_n \coloneqq R/2 u_1$.
If $\gamma_1 < - 8 \gamma_\delta$, define $X = (x_1, \dots, x_n)$ with $x_1 \coloneqq R u_\delta$ and $x_2 = \dots = x_n \coloneqq -R u_\delta$.
This way of defining adversarial inputs does not exactly maximize the local Lipschitz constant for each choice of $n$ and $R$, but leads, for $R$ large enough, to a growth rate of $\sqrt{n}$ for the local Lipschitz constant of $f^{\text{model}}$ (see Figure \ref{fig:real_data}), which is exactly what we need to recover tightness.

\paragraph{Influence of the scaling factor.}
In Figure \ref{fig:real_data}, the scaling factor $R$ is equal to 15.5 and 21.5 for the first two columns, which corresponds to an approximation of the mean radius of real data used with the same models in the first row.
In other words, our adversarial data for the first two models have tokens with a magnitude similar to tokens obtained with real data.
In contrast, for the third column, corresponding to GPT-2 randomly initialized, we take $R = 100$, which is much larger than the magnitude of tokens generated with real data (which is 27.7).
Indeed, we observed that smaller scaling factors induce a growth rate that is slower than $n^{1/2}$.
Studying this aspect more in-depth is an interesting perspective for future work.

\subsection{Discussion}

\autoref{fig:real_data} gives the following insights.
\begin{itemize}[nosep,wide, leftmargin=*]
    \item The Lipschitz constant of self-attention on real data grows significantly with the sequence length, in all considered cases, independently of the dataset, the depth of the attention layer, and of whether self-attention is masked or not.
    The observed growth rate is close to $n^{1/4}$, which is smaller than the worst-case rate $\sqrt{n}$.
    \item The Lipschitz constant of self-attention on our adversarial data grows like $\sqrt{n}$, which is the worst-case rate according to Theorem \ref{thm:unnorm_general_bound}.
    This evidentiates tightness of the bound with respect to the sequence length.
\end{itemize}
Let us make a few remarks.
First, the architecture of BERT adds biases to the traditional formula for self-attention.
This does not affect too much the theoretical analysis (see Appendix \ref{appsubsec:sa_biases}).
Second, the same experiments as in Figure \ref{fig:real_data} performed with GPT-2 pre-trained \cite{radford2019language} lead to a different behavior of the Lipschitz constant.
In particular, the growth rate of the Lipschitz constant can be faster than $\sqrt{n}$, which seems to come from a correlation between the sequence length of inputs and the magnitude of their tokens after going through LayerNorm (see Appendix \ref{appsubsec:gpt}).
Finally, our results point out the difficulty of designing Lispchitz-constrained self-attention layers independently of the sequence length.
Indeed, dividing a self-attention layer by the mean-field bound of Theorem \ref{thm:mean_field_unnorm} to enforce its 1-Lipschitz continuity would induce a dramatic loss of expressive power on smaller sequence lengths.
However, when the sequence length is fixed -- for example with Vision Transformers \cite{dosovitskiy2020image}, dividing the output of the self-attention layer by the bound in Theorem \ref{thm:unnorm_general_bound} is a promising option.

\section*{Conclusion}

In this thorough study of the Lipschitz constant of self-attention, we have identified sharp bounds in different regimes, the most relevant from a practical viewpoint being the general bounds stated in Theorems \ref{thm:unnorm_general_bound} and \ref{thm:masked_general_bound}.
Our theoretical and numerical analyses show that the Lipschitz constant of self-attention grows with the sequence length, the worst-case rate being $\sqrt{n}$, and the rate on real data being at least $n^{1/4}$, and possibly larger for learned positional encoding.
This insight is new and represents an obstruction to designing robust Transformers without modifying the architecture or fixing the sequence length of inputs, which opens interesting avenues for future work.
We have also introduced a novel mean-field framework for masked self-attention, which overcomes the lack of permutation equivariance and paves the way for a study of neural PDEs on Decoders, as \citet{sander2022sinkformers} and \citet{geshkovski2023emergence, geshkovski2023layernorm} do for Encoders.

\section*{Impact Statement}
This paper presents work whose goal is to advance the field of Machine Learning.
There are many potential societal consequences of our work, none of which we feel must be specifically highlighted here.

\bibliographystyle{icml2024}
\bibliography{example_paper}

\begin{thebibliography}{69}
\providecommand{\natexlab}[1]{#1}
\providecommand{\url}[1]{\texttt{#1}}
\expandafter\ifx\csname urlstyle\endcsname\relax
  \providecommand{\doi}[1]{doi: #1}\else
  \providecommand{\doi}{doi: \begingroup \urlstyle{rm}\Url}\fi

\bibitem[Anil et~al.(2019)Anil, Lucas, and Grosse]{anil2019sorting}
Anil, C., Lucas, J., and Grosse, R.
\newblock Sorting out lipschitz function approximation.
\newblock In \emph{International Conference on Machine Learning}, pp.\  291--301. PMLR, 2019.

\bibitem[Ba et~al.(2016)Ba, Kiros, and Hinton]{ba2016layer}
Ba, J.~L., Kiros, J.~R., and Hinton, G.~E.
\newblock Layer normalization.
\newblock \emph{arXiv preprint arXiv:1607.06450}, 2016.

\bibitem[Bahdanau et~al.(2014)Bahdanau, Cho, and Bengio]{bahdanau2014neural}
Bahdanau, D., Cho, K., and Bengio, Y.
\newblock Neural machine translation by jointly learning to align and translate.
\newblock \emph{arXiv preprint arXiv:1409.0473}, 2014.

\bibitem[Bartlett et~al.(2017)Bartlett, Foster, and Telgarsky]{bartlett2017spectrally}
Bartlett, P.~L., Foster, D.~J., and Telgarsky, M.~J.
\newblock Spectrally-normalized margin bounds for neural networks.
\newblock \emph{Advances in neural information processing systems}, 30, 2017.

\bibitem[Behrmann et~al.(2019)Behrmann, Grathwohl, Chen, Duvenaud, and Jacobsen]{behrmann2019invertible}
Behrmann, J., Grathwohl, W., Chen, R.~T., Duvenaud, D., and Jacobsen, J.-H.
\newblock Invertible residual networks.
\newblock In \emph{International conference on machine learning}, pp.\  573--582. PMLR, 2019.

\bibitem[Bird et~al.(2009)Bird, Klein, and Loper]{bird2009natural}
Bird, S., Klein, E., and Loper, E.
\newblock \emph{Natural language processing with Python: analyzing text with the natural language toolkit}.
\newblock " O'Reilly Media, Inc.", 2009.

\bibitem[Bogachev \& Ruas(2007)Bogachev and Ruas]{bogachev2007measure}
Bogachev, V.~I. and Ruas, M. A.~S.
\newblock \emph{Measure theory}, volume~1.
\newblock Springer, 2007.

\bibitem[Brown et~al.(2020)Brown, Mann, Ryder, Subbiah, Kaplan, Dhariwal, Neelakantan, Shyam, Sastry, Askell, et~al.]{brown2020language}
Brown, T., Mann, B., Ryder, N., Subbiah, M., Kaplan, J.~D., Dhariwal, P., Neelakantan, A., Shyam, P., Sastry, G., Askell, A., et~al.
\newblock Language models are few-shot learners.
\newblock \emph{Advances in neural information processing systems}, 33:\penalty0 1877--1901, 2020.

\bibitem[Carlini \& Wagner(2017)Carlini and Wagner]{carlini2017towards}
Carlini, N. and Wagner, D.
\newblock Towards evaluating the robustness of neural networks.
\newblock In \emph{2017 ieee symposium on security and privacy (sp)}, pp.\  39--57. Ieee, 2017.

\bibitem[Catellier et~al.(2023)Catellier, Vaiter, and Garreau]{catellier2023robustness}
Catellier, R., Vaiter, S., and Garreau, D.
\newblock On the robustness of text vectorizers.
\newblock \emph{arXiv preprint arXiv:2303.07203}, 2023.

\bibitem[Chen et~al.(2018)Chen, Rubanova, Bettencourt, and Duvenaud]{chen2018neural}
Chen, R.~T., Rubanova, Y., Bettencourt, J., and Duvenaud, D.~K.
\newblock Neural ordinary differential equations.
\newblock \emph{Advances in neural information processing systems}, 31, 2018.

\bibitem[Chen et~al.(2019)Chen, Behrmann, Duvenaud, and Jacobsen]{chen2019residual}
Chen, R.~T., Behrmann, J., Duvenaud, D.~K., and Jacobsen, J.-H.
\newblock Residual flows for invertible generative modeling.
\newblock \emph{Advances in Neural Information Processing Systems}, 32, 2019.

\bibitem[Cisse et~al.(2017)Cisse, Bojanowski, Grave, Dauphin, and Usunier]{cisse2017parseval}
Cisse, M., Bojanowski, P., Grave, E., Dauphin, Y., and Usunier, N.
\newblock Parseval networks: Improving robustness to adversarial examples.
\newblock In \emph{International conference on machine learning}, pp.\  854--863. PMLR, 2017.

\bibitem[Dasoulas et~al.(2021)Dasoulas, Scaman, and Virmaux]{dasoulas2021lipschitz}
Dasoulas, G., Scaman, K., and Virmaux, A.
\newblock Lipschitz normalization for self-attention layers with application to graph neural networks.
\newblock In \emph{International Conference on Machine Learning}, pp.\  2456--2466. PMLR, 2021.

\bibitem[De~Bie et~al.(2019)De~Bie, Peyr{\'e}, and Cuturi]{debie2019stochastic}
De~Bie, G., Peyr{\'e}, G., and Cuturi, M.
\newblock Stochastic deep networks.
\newblock In \emph{International Conference on Machine Learning}, pp.\  1556--1565. PMLR, 2019.

\bibitem[Dehghani et~al.(2023)Dehghani, Djolonga, Mustafa, Padlewski, Heek, Gilmer, Steiner, Caron, Geirhos, Alabdulmohsin, et~al.]{dehghani2023scaling}
Dehghani, M., Djolonga, J., Mustafa, B., Padlewski, P., Heek, J., Gilmer, J., Steiner, A.~P., Caron, M., Geirhos, R., Alabdulmohsin, I., et~al.
\newblock Scaling vision transformers to 22 billion parameters.
\newblock In \emph{International Conference on Machine Learning}, pp.\  7480--7512. PMLR, 2023.

\bibitem[Devlin et~al.(2018)Devlin, Chang, Lee, and Toutanova]{devlin2018bert}
Devlin, J., Chang, M.-W., Lee, K., and Toutanova, K.
\newblock Bert: Pre-training of deep bidirectional transformers for language understanding.
\newblock \emph{arXiv:1810.04805}, 2018.

\bibitem[Dosovitskiy et~al.(2020)Dosovitskiy, Beyer, Kolesnikov, Weissenborn, Zhai, Unterthiner, Dehghani, Minderer, Heigold, Gelly, et~al.]{dosovitskiy2020image}
Dosovitskiy, A., Beyer, L., Kolesnikov, A., Weissenborn, D., Zhai, X., Unterthiner, T., Dehghani, M., Minderer, M., Heigold, G., Gelly, S., et~al.
\newblock An image is worth 16x16 words: Transformers for image recognition at scale.
\newblock \emph{arXiv:2010.11929}, 2020.

\bibitem[Fazlyab et~al.(2019)Fazlyab, Robey, Hassani, Morari, and Pappas]{fazlyab2019efficient}
Fazlyab, M., Robey, A., Hassani, H., Morari, M., and Pappas, G.
\newblock Efficient and accurate estimation of lipschitz constants for deep neural networks.
\newblock \emph{Advances in Neural Information Processing Systems}, 32, 2019.

\bibitem[Federer(2014)]{federer2014geometric}
Federer, H.
\newblock \emph{Geometric Measure Theory}.
\newblock Classics in Mathematics. Springer Berlin Heidelberg, 2014.
\newblock ISBN 9783642620102.
\newblock URL \url{https://books.google.fr/books?id=jld-BgAAQBAJ}.

\bibitem[Geshkovski et~al.(2023{\natexlab{a}})Geshkovski, Letrouit, Polyanskiy, and Rigollet]{geshkovski2023emergence}
Geshkovski, B., Letrouit, C., Polyanskiy, Y., and Rigollet, P.
\newblock The emergence of clusters in self-attention dynamics.
\newblock \emph{arXiv preprint arXiv:2305.05465}, 2023{\natexlab{a}}.

\bibitem[Geshkovski et~al.(2023{\natexlab{b}})Geshkovski, Letrouit, Polyanskiy, and Rigollet]{geshkovski2023layernorm}
Geshkovski, B., Letrouit, C., Polyanskiy, Y., and Rigollet, P.
\newblock A mathematical perspective on transformers.
\newblock 2023{\natexlab{b}}.

\bibitem[Goodfellow et~al.(2014)Goodfellow, Shlens, and Szegedy]{goodfellow2014explaining}
Goodfellow, I.~J., Shlens, J., and Szegedy, C.
\newblock Explaining and harnessing adversarial examples.
\newblock \emph{arXiv preprint arXiv:1412.6572}, 2014.

\bibitem[Gupta \& Verma(2023)Gupta and Verma]{gupta2023certvit}
Gupta, K. and Verma, S.
\newblock Certvit: Certified robustness of pre-trained vision transformers.
\newblock \emph{arXiv preprint arXiv:2302.10287}, 2023.

\bibitem[Hein \& Andriushchenko(2017)Hein and Andriushchenko]{hein2017formal}
Hein, M. and Andriushchenko, M.
\newblock Formal guarantees on the robustness of a classifier against adversarial manipulation.
\newblock \emph{Advances in neural information processing systems}, 30, 2017.

\bibitem[Hosseini et~al.(2023)Hosseini, Hsu, and Taghvaei]{hosseini2023conditional}
Hosseini, B., Hsu, A.~W., and Taghvaei, A.
\newblock Conditional optimal transport on function spaces.
\newblock \emph{arXiv preprint arXiv:2311.05672}, 2023.

\bibitem[Jia et~al.(2023)Jia, Chen, Mao, Duan, Gu, Zhang, Xue, and Cao]{jia2023revisiting}
Jia, X., Chen, Y., Mao, X., Duan, R., Gu, J., Zhang, R., Xue, H., and Cao, X.
\newblock Revisiting and exploring efficient fast adversarial training via law: Lipschitz regularization and auto weight averaging.
\newblock \emph{arXiv preprint arXiv:2308.11443}, 2023.

\bibitem[Jiang et~al.(2023)Jiang, Sablayrolles, Mensch, Bamford, Chaplot, Casas, Bressand, Lengyel, Lample, Saulnier, et~al.]{jiang2023mistral}
Jiang, A.~Q., Sablayrolles, A., Mensch, A., Bamford, C., Chaplot, D.~S., Casas, D. d.~l., Bressand, F., Lengyel, G., Lample, G., Saulnier, L., et~al.
\newblock Mistral 7b.
\newblock \emph{arXiv preprint arXiv:2310.06825}, 2023.

\bibitem[Kim et~al.(2021)Kim, Papamakarios, and Mnih]{kim2021lipschitz}
Kim, H., Papamakarios, G., and Mnih, A.
\newblock The lipschitz constant of self-attention, 2021.

\bibitem[Kurakin et~al.(2016)Kurakin, Goodfellow, and Bengio]{kurakin2016adversarial}
Kurakin, A., Goodfellow, I., and Bengio, S.
\newblock Adversarial machine learning at scale.
\newblock \emph{arXiv preprint arXiv:1611.01236}, 2016.

\bibitem[Latorre et~al.(2020)Latorre, Rolland, and Cevher]{latorre2020lipschitz}
Latorre, F., Rolland, P., and Cevher, V.
\newblock Lipschitz constant estimation of neural networks via sparse polynomial optimization.
\newblock \emph{arXiv preprint arXiv:2004.08688}, 2020.

\bibitem[Lee et~al.(2019)Lee, Lee, Kim, Kosiorek, Choi, and Teh]{lee2019set}
Lee, J., Lee, Y., Kim, J., Kosiorek, A., Choi, S., and Teh, Y.~W.
\newblock Set transformer: A framework for attention-based permutation-invariant neural networks.
\newblock In \emph{International conference on machine learning}, pp.\  3744--3753. PMLR, 2019.

\bibitem[Leino et~al.(2021)Leino, Wang, and Fredrikson]{leino2021globally}
Leino, K., Wang, Z., and Fredrikson, M.
\newblock Globally-robust neural networks.
\newblock In \emph{International Conference on Machine Learning}, pp.\  6212--6222. PMLR, 2021.

\bibitem[Likhosherstov et~al.(2021)Likhosherstov, Choromanski, and Weller]{likhosherstov2021expressive}
Likhosherstov, V., Choromanski, K., and Weller, A.
\newblock On the expressive power of self-attention matrices.
\newblock \emph{arXiv preprint arXiv:2106.03764}, 2021.

\bibitem[Liu et~al.(2018)Liu, Saleh, Pot, Goodrich, Sepassi, Kaiser, and Shazeer]{liu2018generating}
Liu, P.~J., Saleh, M., Pot, E., Goodrich, B., Sepassi, R., Kaiser, L., and Shazeer, N.
\newblock Generating wikipedia by summarizing long sequences.
\newblock \emph{arXiv:1801.10198}, 2018.

\bibitem[Lu et~al.(2019)Lu, Li, He, Sun, Dong, Qin, Wang, and Liu]{lu2019understanding}
Lu, Y., Li, Z., He, D., Sun, Z., Dong, B., Qin, T., Wang, L., and Liu, T.-Y.
\newblock Understanding and improving transformer from a multi-particle dynamic system point of view.
\newblock \emph{arXiv preprint arXiv:1906.02762}, 2019.

\bibitem[Miyato et~al.(2015)Miyato, Maeda, Koyama, Nakae, and Ishii]{miyato2015distributional}
Miyato, T., Maeda, S.-i., Koyama, M., Nakae, K., and Ishii, S.
\newblock Distributional smoothing with virtual adversarial training.
\newblock \emph{arXiv preprint arXiv:1507.00677}, 2015.

\bibitem[Miyato et~al.(2018)Miyato, Kataoka, Koyama, and Yoshida]{miyato2018spectral}
Miyato, T., Kataoka, T., Koyama, M., and Yoshida, Y.
\newblock Spectral normalization for generative adversarial networks.
\newblock \emph{arXiv preprint arXiv:1802.05957}, 2018.

\bibitem[Moosavi-Dezfooli et~al.(2016)Moosavi-Dezfooli, Fawzi, and Frossard]{moosavi2016deepfool}
Moosavi-Dezfooli, S.-M., Fawzi, A., and Frossard, P.
\newblock Deepfool: a simple and accurate method to fool deep neural networks.
\newblock In \emph{Proceedings of the IEEE conference on computer vision and pattern recognition}, pp.\  2574--2582, 2016.

\bibitem[Neyshabur et~al.(2017)Neyshabur, Bhojanapalli, McAllester, and Srebro]{neyshabur2017exploring}
Neyshabur, B., Bhojanapalli, S., McAllester, D., and Srebro, N.
\newblock Exploring generalization in deep learning.
\newblock \emph{Advances in neural information processing systems}, 30, 2017.

\bibitem[OpenAI(2023)]{openai2023gpt4}
OpenAI.
\newblock Gpt-4 technical report, 2023.

\bibitem[Papernot et~al.(2016)Papernot, McDaniel, Wu, Jha, and Swami]{papernot2016distillation}
Papernot, N., McDaniel, P., Wu, X., Jha, S., and Swami, A.
\newblock Distillation as a defense to adversarial perturbations against deep neural networks.
\newblock In \emph{2016 IEEE symposium on security and privacy (SP)}, pp.\  582--597. IEEE, 2016.

\bibitem[Pevny \& Kovarik(2019)Pevny and Kovarik]{pevny2019approximation}
Pevny, T. and Kovarik, V.
\newblock Approximation capability of neural networks on spaces of probability measures and tree-structured domains.
\newblock \emph{arXiv preprint arXiv:1906.00764}, 2019.

\bibitem[Peyr{\'e} et~al.(2019)Peyr{\'e}, Cuturi, et~al.]{peyre2019computational}
Peyr{\'e}, G., Cuturi, M., et~al.
\newblock Computational optimal transport: With applications to data science.
\newblock \emph{Foundations and Trends{\textregistered} in Machine Learning}, 11\penalty0 (5-6):\penalty0 355--607, 2019.

\bibitem[Qi et~al.(2023)Qi, Wang, Chen, Shi, and Zhang]{qi2023lipsformer}
Qi, X., Wang, J., Chen, Y., Shi, Y., and Zhang, L.
\newblock Lipsformer: Introducing lipschitz continuity to vision transformers.
\newblock \emph{arXiv preprint arXiv:2304.09856}, 2023.

\bibitem[Radford et~al.(2019)Radford, Wu, Child, Luan, Amodei, Sutskever, et~al.]{radford2019language}
Radford, A., Wu, J., Child, R., Luan, D., Amodei, D., Sutskever, I., et~al.
\newblock Language models are unsupervised multitask learners.
\newblock \emph{OpenAI blog}, 1\penalty0 (8):\penalty0 9, 2019.

\bibitem[Rosca et~al.(2020)Rosca, Weber, Gretton, and Mohamed]{rosca2020case}
Rosca, M., Weber, T., Gretton, A., and Mohamed, S.
\newblock A case for new neural network smoothness constraints.
\newblock \emph{PMLR}, 2020.

\bibitem[Sander et~al.(2022)Sander, Ablin, Blondel, and Peyr{\'e}]{sander2022sinkformers}
Sander, M.~E., Ablin, P., Blondel, M., and Peyr{\'e}, G.
\newblock Sinkformers: Transformers with doubly stochastic attention.
\newblock In \emph{International Conference on Artificial Intelligence and Statistics}, pp.\  3515--3530. PMLR, 2022.

\bibitem[Santambrogio(2015)]{santambrogio2015optimal}
Santambrogio, F.
\newblock Optimal transport for applied mathematicians.
\newblock \emph{Birk{\"a}user, NY}, 55\penalty0 (58-63):\penalty0 94, 2015.

\bibitem[Sokoli{\'c} et~al.(2017)Sokoli{\'c}, Giryes, Sapiro, and Rodrigues]{sokolic2017robust}
Sokoli{\'c}, J., Giryes, R., Sapiro, G., and Rodrigues, M.~R.
\newblock Robust large margin deep neural networks.
\newblock \emph{IEEE Transactions on Signal Processing}, 65\penalty0 (16):\penalty0 4265--4280, 2017.

\bibitem[Sra et~al.(2012)Sra, Nowozin, and Wright]{sra2012optimization}
Sra, S., Nowozin, S., and Wright, S.~J.
\newblock \emph{Optimization for machine learning}.
\newblock Mit Press, 2012.

\bibitem[Szegedy et~al.(2013)Szegedy, Zaremba, Sutskever, Bruna, Erhan, Goodfellow, and Fergus]{szegedy2013intriguing}
Szegedy, C., Zaremba, W., Sutskever, I., Bruna, J., Erhan, D., Goodfellow, I., and Fergus, R.
\newblock Intriguing properties of neural networks.
\newblock \emph{arXiv preprint arXiv:1312.6199}, 2013.

\bibitem[Touvron et~al.(2023)Touvron, Martin, Stone, Albert, Almahairi, Babaei, Bashlykov, Batra, Bhargava, Bhosale, et~al.]{touvron2023llama}
Touvron, H., Martin, L., Stone, K., Albert, P., Almahairi, A., Babaei, Y., Bashlykov, N., Batra, S., Bhargava, P., Bhosale, S., et~al.
\newblock Llama 2: Open foundation and fine-tuned chat models.
\newblock \emph{arXiv preprint arXiv:2307.09288}, 2023.

\bibitem[Tsuzuku et~al.(2018)Tsuzuku, Sato, and Sugiyama]{tsuzuku2018lipschitz}
Tsuzuku, Y., Sato, I., and Sugiyama, M.
\newblock Lipschitz-margin training: Scalable certification of perturbation invariance for deep neural networks.
\newblock \emph{Advances in neural information processing systems}, 31, 2018.

\bibitem[Vaswani et~al.(2017)Vaswani, Shazeer, Parmar, Uszkoreit, Jones, Gomez, Kaiser, and Polosukhin]{vaswani2017attention}
Vaswani, A., Shazeer, N., Parmar, N., Uszkoreit, J., Jones, L., Gomez, A.~N., Kaiser, {\L}., and Polosukhin, I.
\newblock Attention is all you need.
\newblock \emph{Advances in neural information processing systems}, 30, 2017.

\bibitem[Virmaux \& Scaman(2018)Virmaux and Scaman]{virmaux2018lipschitz}
Virmaux, A. and Scaman, K.
\newblock Lipschitz regularity of deep neural networks: analysis and efficient estimation.
\newblock \emph{Advances in Neural Information Processing Systems}, 31, 2018.

\bibitem[von Luxburg \& Bousquet(2004)von Luxburg and Bousquet]{von2004distance}
von Luxburg, U. and Bousquet, O.
\newblock Distance-based classification with lipschitz functions.
\newblock \emph{J. Mach. Learn. Res.}, 5\penalty0 (Jun):\penalty0 669--695, 2004.

\bibitem[Vuckovic et~al.(2020)Vuckovic, Baratin, and des Combes]{Vuckovic2020AMT}
Vuckovic, J., Baratin, A., and des Combes, R.~T.
\newblock A mathematical theory of attention.
\newblock \emph{ArXiv}, abs/2007.02876, 2020.

\bibitem[Vuckovic et~al.(2021)Vuckovic, Baratin, and Combes]{vuckovic2021regularity}
Vuckovic, J., Baratin, A., and Combes, R. T.~d.
\newblock On the regularity of attention.
\newblock \emph{arXiv preprint arXiv:2102.05628}, 2021.

\bibitem[Vyas et~al.(2020)Vyas, Katharopoulos, and Fleuret]{vyas2020fast}
Vyas, A., Katharopoulos, A., and Fleuret, F.
\newblock Fast transformers with clustered attention.
\newblock \emph{Advances in Neural Information Processing Systems}, 33:\penalty0 21665--21674, 2020.

\bibitem[Weng et~al.(2018)Weng, Zhang, Chen, Yi, Su, Gao, Hsieh, and Daniel]{weng2018evaluating}
Weng, T.-W., Zhang, H., Chen, P.-Y., Yi, J., Su, D., Gao, Y., Hsieh, C.-J., and Daniel, L.
\newblock Evaluating the robustness of neural networks: An extreme value theory approach.
\newblock \emph{arXiv preprint arXiv:1801.10578}, 2018.

\bibitem[Wolf et~al.(2019)Wolf, Debut, Sanh, Chaumond, Delangue, Moi, Cistac, Rault, Louf, Funtowicz, et~al.]{wolf2019huggingface}
Wolf, T., Debut, L., Sanh, V., Chaumond, J., Delangue, C., Moi, A., Cistac, P., Rault, T., Louf, R., Funtowicz, M., et~al.
\newblock Huggingface's transformers: State-of-the-art natural language processing.
\newblock \emph{arXiv preprint arXiv:1910.03771}, 2019.

\bibitem[Xiong et~al.(2020)Xiong, Yang, He, Zheng, Zheng, Xing, Zhang, Lan, Wang, and Liu]{xiong2020layernorm}
Xiong, R., Yang, Y., He, D., Zheng, K., Zheng, S., Xing, C., Zhang, H., Lan, Y., Wang, L., and Liu, T.
\newblock On layer normalization in the transformer architecture.
\newblock In \emph{International Conference on Machine Learning}, pp.\  10524--10533. PMLR, 2020.

\bibitem[Ye et~al.(2023)Ye, Ma, Cao, and Tang]{ye2023mitigating}
Ye, W., Ma, Y., Cao, X., and Tang, K.
\newblock Mitigating transformer overconfidence via lipschitz regularization.
\newblock \emph{arXiv preprint arXiv:2306.06849}, 2023.

\bibitem[Zhai et~al.(2022)Zhai, Kolesnikov, Houlsby, and Beyer]{zhai2022scaling}
Zhai, X., Kolesnikov, A., Houlsby, N., and Beyer, L.
\newblock Scaling vision transformers.
\newblock In \emph{Proceedings of the IEEE/CVF Conference on Computer Vision and Pattern Recognition}, pp.\  12104--12113, 2022.

\bibitem[Zhang \& Sennrich(2019)Zhang and Sennrich]{zhang2019root}
Zhang, B. and Sennrich, R.
\newblock Root mean square layer normalization.
\newblock \emph{Advances in Neural Information Processing Systems}, 32, 2019.

\bibitem[Zhang et~al.(2015)Zhang, Zhao, and LeCun]{Zhang2015CharacterlevelCN}
Zhang, X., Zhao, J.~J., and LeCun, Y.
\newblock Character-level convolutional networks for text classification.
\newblock In \emph{NIPS}, 2015.

\bibitem[Zhao et~al.(2020)Zhao, Jiang, Jia, Torr, and Koltun]{zhao2020point}
Zhao, H., Jiang, L., Jia, J., Torr, P., and Koltun, V.
\newblock Point transformer. arxiv.
\newblock \emph{arXiv preprint arXiv:2012.09164}, 2020.

\bibitem[Zweig \& Bruna(2021)Zweig and Bruna]{zweig2021functional}
Zweig, A. and Bruna, J.
\newblock A functional perspective on learning symmetric functions with neural networks.
\newblock In \emph{International Conference on Machine Learning}, pp.\  13023--13032. PMLR, 2021.

\end{thebibliography}

\newpage
\appendix
\onecolumn

\section{Optimal Transport Toolbox}
\label{appsec:optimaltransport}

This section gathers some useful definitions and lemmas from optimal transport.
In what follows, $\mathcal{X}$ is a Borel subset of $\R^d$.

\subsection{Pushforward, Wasserstein Distance}

Let us start with the notion of pushforward.

\begin{definition}[Pushforward]
    Set $\mu$ a probability measure on $\mathcal{X}$ and $\varphi\colon \mathcal{X} \to \mathcal{X}$ a measurable function.
    The pushforward of $\mu$ by $\varphi$, denoted $\varphi_\sharp \mu$, is the probability measure given by
    $$\left (\varphi_\sharp \mu\right )(B) = \mu(\varphi^{-1}(B))$$
    for any Borel set $B\subset \mathcal{X}$, where $\varphi^{-1}(B) \coloneqq \{ x\in \R^d : \varphi(x) \in B \}$.
\end{definition}

The pushforward measure $\varphi_\sharp \mu$ can be seen as the result of a transportation of the mass of $\mu$ by $\varphi$.
Intuitively, $\varphi_\sharp \mu$ is obtained by transporting each element of mass $\mu(\dd x)$ from $x$ to $\varphi(x)$.

Another crucial tool is the notion of Wasserstein distance.

\begin{definition}[Wasserstein space, Wasserstein distance]
    Let $p\ge 1$.
    Denote
    $$\mathcal{P}_p(\mathcal{X}) \coloneqq \{ \mu \in  \mathcal{P}(\mathcal{X}) : \int_\mathcal{X}  \lvert x\rvert^p \dd \mu(x) < \infty\}$$
    the $p$-Wasserstein space.
    Then, the $p$-Wasserstein distance between two probability measures $\mu, \nu \in \mathcal{P}_p(\mathcal{X})$ is defined as
    $$W_p(\mu, \nu) \coloneqq \left (\inf_{\pi \in \Pi(\mu, \nu)} \int \lvert x-y\rvert ^p \dd \pi(x,y)\right )^{1/p}$$
    with $\Pi(\mu, \nu)$ the set of all couplings between $\mu$ and $\nu$, i.e. of all probability measures $\pi \in \mathcal{P}(\mathcal{X} \times \mathcal{X})$ such that $\int \pi(\cdot,y) \dd y = \mu$ and $\int \pi(x,\cdot) \dd x = \nu$.
\end{definition}

Wasserstein distances have the following nice property, which is a direct consequence of Jensen inequality.

\begin{lemma}
    \label{applem:ineq_w_distances}
    For every $p\ge 1$, it holds
    $$W_1\le W_p.$$
\end{lemma}

The distance $W_1$ has also a simple dual formulation.

\begin{lemma}[$W_1$ duality formula]
    \label{applem:duality_formula}
    The distance $W_1$ can be rewritten with the so-called duality formula: for all $\mu, \nu \in \mathcal{P}_1(\mathcal{X})$, it holds
    \begin{equation}
        \label{appeq:duality}
        W_1(\mu, \nu) = \sup_{\lip(\varphi)\le 1} \int \varphi\, \dd (\mu - \nu),
    \end{equation}
    where the supremum is taken over all functions $\varphi\colon \mathcal{X} \to \R$ with a Lipschitz constant bounded by one.
\end{lemma}

The following result is useful to bound the Wasserstein distance between two probability measures that are pushed forward by the same map.

\begin{lemma}
    \label{applem:lip_wass1}
    Let $p\ge 1$.
    Consider a measurable function $\varphi\colon \mathcal{X} \to \mathcal{X}$, and probability measures $\mu, \nu \in \mathcal{P}_p(\mathcal{X})$ such that $\varphi_\sharp \mu \in \mathcal{P}_p(\mathcal{X})$ and $ \varphi_\sharp \nu \in \mathcal{P}_p(\mathcal{X})$.
    Then, it holds
    $$W_p(\varphi_\sharp \mu, \varphi_\sharp \nu) \le \lip(\varphi) W_p(\mu, \nu).$$
\end{lemma}

\begin{proof}
    We have
    \begin{align*}
        W_p(\varphi_\sharp \mu, \varphi_\sharp \nu)^p &= \inf_{\pi' \in \Pi(\varphi_\sharp \mu, \varphi_\sharp \nu)} \int \norm{x-y}^p \dd \pi'(x,y) \\
        &\le \inf_{\pi \in \Pi(\mu, \nu)} \int \norm{\varphi(x)-\varphi(y)}^p \dd \pi(x,y) \\
        &\le \lip(\varphi)^p \inf_{\pi \in \Pi(\mu, \nu)} \int \norm{x-y}^p \dd \pi(x,y)\\
        &= \lip(\varphi)^p W_p(\mu, \nu)^p,
    \end{align*}
    where the first inequality derives from the fact that every $\pi \in \Pi(\mu, \nu)$ induces a coupling $\pi' \in \Pi(\varphi_\sharp \mu, \varphi_\sharp \nu)$ by setting
    $$\pi'(B_1 \times B_2) \coloneqq \pi(\varphi^{-1}(B_1)\times \varphi^{-1}(B_2))$$
    for all Borel sets $B_1,B_2\subset \mathcal{X}$, and that with this choice of $\pi'$ it holds
    $$ \int \norm{x-y}^p \dd \pi'(x,y) = \int \norm{\varphi(x)-\varphi(y)}^p \dd \pi(x,y).$$
\end{proof}

Let us now bound the Wasserstein distance between two different pushforwards of the same probability measure.

\begin{lemma}
    \label{applem:lip_wass2}
    Let $p\ge 1$.
    Consider two measurable functions $\varphi,\psi\colon \mathcal{X} \to \mathcal{X}$, and a probability measure $\mu \in \mathcal{P}_p(\mathcal{X})$ such that $\varphi_\sharp \mu \in \mathcal{P}_p(\mathcal{X})$ and $\psi_\sharp \mu \in \mathcal{P}_p(\mathcal{X})$.
    Then, it holds
    $$W_p(\varphi_\sharp \mu, \psi_\sharp \mu) \le \norm{\varphi-\psi}_{L^p(\mu)}.$$
\end{lemma}

\begin{proof}
    Recall that
    \begin{equation*}
        W_p(\varphi_\sharp \mu, \psi_\sharp \mu)^p = \inf_{\pi' \in \Pi(\varphi_\sharp \mu, \psi_\sharp \mu)} \int \norm{x-y}^p \dd \pi'(x,y).
    \end{equation*}
    Now consider the following coupling between $\varphi_\sharp \mu$ and $\psi_\sharp \mu$, defined by the relation
    $$\pi'(B\times  C) \coloneqq \mu(\varphi^{-1}(B)\cap \psi^{-1}(C))$$
    for every Borel sets $B, C\subset \mathcal{X}$.
    In other words, we set $\dd \pi'(y, z) \coloneqq \int_{\varphi^{-1}(y)\cap \psi^{-1}(z)}\dd \mu$, and $\dd \pi'(y, z) = 0$ if $\varphi^{-1}(y)\cap \psi^{-1}(z) = \emptyset$.
    With this definition of $\pi'$, we have
    \begin{align*}
        W_p(\varphi_\sharp \mu, \psi_\sharp \mu)^p &\le \int\norm{x-y}^p \dd \pi'(x,y) 
        = \int \norm{\varphi(x) - \psi(x)}^p \dd \mu(x).
    \end{align*}
\end{proof}

\subsection{Geodesics}

The notion of a geodesic is useful for the following section of the Appendix.

\begin{definition}[Geodesic]
    Let $(\mathcal{E}, d_{\mathcal{E}})$ be a metric space.
    A curve $\gamma\colon [0,1] \to \mathcal{E}$ is called a geodesic if there exists a constant $v\ge 0$ such that for all $ t_1, t_2\in [0, 1]$ we have
    $$d_\mathcal{E}(\gamma(t_1), \gamma(t_2)) = v\lvert t_2 - t_1\rvert.$$
    We say that the space $\mathcal{E}$ is geodesic if for any $x, y\in \mathcal{E}$, there exists a geodesic between $x$ and $y$.
\end{definition}

One important example of geodesic space is the 2-Wasserstein space $\mathcal{P}_2(\R^d)$.

\begin{lemma}[\citet{santambrogio2015optimal}]
    The space $\mathcal{P}_2(\R^d)$ is a geodesic space.
\end{lemma}

\section{Local Lipschitz Constant in the Mean-Field Setting}

\subsection{Mean-Field Self-Attention Generalizes Self-Attention}
\label{appsubsec:mean_field_attention}

For any input $X \in (\R^d)^n$, we have $F(m(X)) = m(f(X))$.
Indeed, we can rewrite $f$ as
\begin{equation} 
\label{appeq:discrete_pushforward}
    f(X) = (\Gamma_X(x_1),\dots, \Gamma_X(x_n)),
\end{equation}
with
$$\Gamma_X\colon x\mapsto \frac{\sum_{i=1}^n \exp \left (\frac{1}{\sqrt{k}}x^\top Q^\top K x_i\right )Vx_i}{\sum_{i=1}^n \exp \left (\frac{1}{\sqrt{k}} x^\top Q^\top K x_i\right )}.$$
Seeing the ratio of sums as a ratio of integrals against the empirical measure $m(X)$, and Equation (\ref{appeq:discrete_pushforward}) as a pushforward leads precisely to the formula of mean-field self-attention.

\subsection{Local Lipschitz Constant in a General Metric Framework}
\label{appsubsec:general_local_lip_cst}

The local Lipschitz constant can be defined for any function between two metric spaces, as follows.

\begin{definition}[Local Lipschitz constant]
\label{def:loc_cst}
    Let $\varphi\colon \mathcal{E} \to \mathcal{F}$ be a map between two metric spaces.
    We define the local Lipschitz constant of $\varphi$ at point $x\in \mathcal{E}$ as
    $$\lip_x(\varphi) \coloneqq \lim_{\varepsilon \to 0^+} \lip \varphi_{\lvert B(x, \varepsilon)}.$$
    The limit exists, as $\lip \varphi_{\lvert B(x, \varepsilon)}$ decreases with $\varepsilon$.
\end{definition}

Definition \ref{def:loc_cst} is interesting, as it captures more information than the global Lipschitz constant.
More precisely, we have the following connection between the two notions.

\begin{lemma}
\label{lem:loc_glob_lip}
Let $\varphi\colon \mathcal{E} \to \mathcal{F}$ be a map between two metric spaces.
We have
$$\lip (\varphi) \ge \sup_{x\in \mathcal{E}} \lip_x (\varphi).$$
Assume moreover that the space $\mathcal{E}$ admits geodesics, which is the case for $\R^d$ and $\mathcal{P}_2(\R^d)$ equipped with $W_2$ (see \ref{appsec:optimaltransport}).
Then, this inequality becomes an equality.
\end{lemma}

\subsection{Proof of Lemma \ref{lem:link_euclidian_meanfield}}
\label{appsubsec:link_euclidian_meanfield}

Let us prove the following slightly stronger result.

\begin{lemma}
    Let $p\ge 1$.
    For any matrix $X\in \R^{n\times d}$, we have
    $$\lip_{X}^{\norm{\cdot}_{F,p}}(f) = \lip_{m(X)}^{W_p}(F_{\lvert \mathcal{M}_n(\R^d)})\le \lip_{m(X)}^{W_p}(F).$$
\end{lemma}

\begin{proof}
    Set $X\in \R^{n\times d}$.
    One can choose $\varepsilon_1 > 0$ small enough (for example $\varepsilon_1 < \min_{x_i\neq x_j}\norm{x_i-x_j}/2$) to have
    $$\norm{X-Y}_{F,p}\le \varepsilon_1 \Rightarrow \norm{X-Y}_{F,p} = W_p(m(X),m(Y)).$$ 
    Indeed, if 
    $$\norm{X-Y}_{F,p}\le \varepsilon_1 < \min_{x_i\neq x_j}\norm{x_i-x_j}/2,$$ 
    then for all $i \in \{1, \dots, n\}$, we have $\norm{x_i-y_i}\le \varepsilon_1$, and thus $x_i$ is the nearest neighbor (or one of the nearest neighbors) of $y_i$ among the $x_j$.
    Similarly, one can choose $\varepsilon_2 > 0$ small enough to have
    $$\norm{f(X)-f(Y)}_{F,p}\le \varepsilon_2 \Rightarrow \norm{f(X)-f(Y)}_{F,p} = W_p(m(f(X)),m(f(Y))).$$
    Then, we can set $\varepsilon\le \varepsilon_1$ small enough to have
    $$ \norm{X-Y}_{F,p}\le \varepsilon \Rightarrow \norm{f(X)-f(Y)}_{F,p}\le \varepsilon_2,$$
    and it holds, for all $\eta\le \varepsilon$ and all $Y$ such that $\norm{X-Y}_{F,p}\le \eta$, that 
    $$\norm{X-Y}_{F,p} = W_p(m(X),m(Y))$$ and $$\norm{f(X)-f(Y)}_{F,p} = W_p(m(f(X)),m(f(Y))).$$
    Now for all $\eta \le \varepsilon$, we have
    \begin{align*}
        \lip^{\norm{\cdot}_{F,p}}f_{\lvert B_{\norm{\cdot}_{F,p}}(X,\eta)} &= \sup_{Y \in B_{\norm{\cdot}_{F,p}}(X,\eta)}\frac{\norm{f(X)-f(Y)}_{F,p}}{\norm{X-Y}_{F,p}}\\
        &= \sup_{Y \in B_{\norm{\cdot}_{F,p}}(X,\eta)}\frac{W_p(m(f(X)), m(f(Y)))}{W_p(m(X),m(Y))} \\
        & = \sup_{Y \in B_{\norm{\cdot}_{F,p}}(X,\eta)}\frac{W_p(F(m(X)), F(m(Y)))}{W_p(m(X),m(Y))},
    \end{align*}
    by definition of $\varepsilon$ and $F$.
    We conclude the proof by noticing that:
    \begin{itemize}
        \item $Y\in B_{\norm{\cdot}_{F,p}}(X, \eta)$ implies $m(Y)\in B_{W_p}(m(X), \eta)$, which shows that $\lip^{\norm{\cdot}_{F,p}}f_{\lvert B_{\norm{\cdot}_{F,p}}(X,\eta)}\le \lip^{W_p}F_{\lvert \mathcal{M}_n(\R^d)}$,
        \item $\mu \in B_{W_p}(m(X), \eta)$ with $\mu\in \mathcal{M}_n(\R^d)$ implies the existence of $Y\in \R^{n\times d}$ such that $\mu = m(Y)$ and $\norm{X-Y}_{F,p} = W_p(m(X), m(Y))$, so that $Y\in B_{\norm{\cdot}_{F,p}}(X, \eta)$.
        Indeed, take $Y$ such that $\mu_n = m(Y)$ and then permute its coordinates so that $x_i$ becomes the nearest neighbor (or one of the nearest neighbors) of $y_i$.
        This shows the reverse inequality and concludes the proof.
    \end{itemize}
\end{proof}

\section{Proofs of Section \ref{sec:bounds}}
\label{appsec:proofs_sec3}

We have the following useful Lemma.

\begin{lemma}
    \label{applem:bound_variance}
    Let $\mu$ be a probability measure supported in $\bbar(x_0,R) \subset \R^d$, with any $x_0\in \R^d$ and $R>0$.
    Then, denoting $\var \mu \coloneqq \mathbb E [(Z - \mathbb E Z)(Z - \mathbb E Z)^\top]$ with $Z$ a random variable distributed according to $\mu$, we have
    $$\norm{\var \mu}_2 \le R^2,$$
    with equality when $\mu = \frac{1}{2}(\delta_{x_0+x} + \delta_{x_0-x})$ for any $x \in \R^d$ such that $\lvert x \rvert = R$.
\end{lemma}
    
\begin{proof}
    Let us assume without loss of generality that $x_0 = 0$.
    It is straightforward to check that if $\mu = \frac{1}{2}(\delta_x + \delta_{-x})$ then $\norm{\var \mu}_2 = R^2$.
    To show that this is the maximal value the variance can take, we use the triangle inequality:
        \begin{align*}
            \norm{\mathbb E [(Z - \mathbb E Z)(Z - \mathbb E Z)^\top]}_2 &\le \mathbb E \norm{(Z - \mathbb E Z)(Z - \mathbb E Z)^\top}_2 \\
            &= \mathbb E [\lvert Z - \mathbb E Z\rvert^2]\\
            &= \mathbb E [\lvert Z\rvert^2] - \lvert \mathbb E Z\rvert^2.
        \end{align*}
    Now let us pick any $x \in B_R \setminus B(0,r)$.
    We have $\mathbb E (Z - x)^\top (Z + x) \le 0$, as the angle between the vectors $Z - x$ and $Z + x$ is at least $\pi/2$ for $Z$ in $B(0,R)$.
    By expanding this relationship we get
    $$E[\lvert Z\rvert^2] - \lvert x\rvert^2 \le 0,$$
    which yields the result.
\end{proof}

\subsection{The Jacobian of Self-Attention}
\label{appsubsec:jacobian_formula}

\begin{lemma}
    \label{applem:jacobian_formula}
    Let $X = (x_1, \dots, x_n) \in (\R^d)^n$.
    For all perturbations $\epsilon \coloneqq (\epsilon_1,\dots, \epsilon_n) \in (\R^d)^n$ and all $i \in \{1, \dots, n\}$, we have
    \begin{equation*}
    (D_{X}f)(\epsilon)_i = V \sum_{j=1}^n P_{ij} (x_j - \sum_{k=1}^n P_{ik}x_k)x_i^\top A^\top \epsilon_j + V\sum_{j=1}^n P_{ij} \epsilon_j + V \sum_{j=1}^n P_{ij}(x_j - \sum_{k=1}^n P_{ik}x_k) x_j^\top A \epsilon_i,
    \end{equation*}
    with $P_{ij} \coloneqq e^{ x_i^\top A^\top x_j} / \sum_{k=1}^n e^{ x_i^\top A^\top x_k}$.
\end{lemma}

\subsection{Proof of Theorem \ref{thm:unnorm_general_bound}}
\label{appsubsec:unnorm_general_bound}

Let $X\in B_R^n$ and $\epsilon \coloneqq (\epsilon_1,\dots, \epsilon_n) \in (\R^d)^n$ such that $\norm{\epsilon}_F = 1$.
According to Lemma \ref{applem:jacobian_formula}, for all $i \in \{1, \dots, n\}$ we have
\begin{equation*}
(D_{X}f)(\epsilon)_i = V \sum_{j=1}^n P_{ij} (x_j - \sum_{k=1}^n P_{ik}x_k)x_i^\top A^\top \epsilon_j + V\sum_{j=1}^n P_{ij} \epsilon_j + V \sum_{j=1}^n P_{ij}(x_j - \sum_{k=1}^n P_{ik}x_k) x_j^\top A \epsilon_i,
\end{equation*}
with $P_{ij} \coloneqq e^{x_i^\top A^\top x_j} / \sum_{k=1}^n e^{ x_i^\top A^\top x_k}$.
The triangle inequality gives
\begin{equation*}
    \lvert (D_{X}f)(\epsilon)_i \rvert \le \norm{V}_2 \left ( \norm{A}_2 R \left \lvert \sum_{j=1}^n P_{ij} (x_j - \sum_{k=1}^n P_{ik} x_k) \epsilon_j\right \rvert + \left \lvert \sum_{j=1}^n P_{ij} \epsilon_j \right \rvert + \norm{A}_2 \norm{\mathrm{Var}^{(i)}}_2 \lvert \epsilon_i  \rvert \right )
\end{equation*}
where $\mathrm{Var}^{(i)} \coloneqq \sum_{j=1}^n P_{ij}(x_j - \sum_{k=1}^n P_{ik}x_k) x_j^\top$ is the variance of the probability measure $\sum_{j=1}^n P_{ij} \delta_{x_j}$.
Lemma \ref{applem:bound_variance} gives us that $\norm{\mathrm{Var}^{(i)}}_2 \le R^2$.
We can also apply Cauchy-Schwarz inequality to get
$$\left \lvert \sum_{j=1}^n P_{ij} (x_j - \sum_{k=1}^n P_{ik} x_k) \epsilon_j\right \rvert \le \left ( \sum_{j=1}^n P_{ij} \left \lvert x_j - \sum_{k=1}^n P_{ik} x_k \right \rvert^2 \right )^{1/2} \left (\sum_{j=1}^n P_{ij} \lvert \epsilon_j \rvert^2 \right )^{1/2}.$$
The proof of Lemma \ref{applem:bound_variance} allows us to bound
$$\sum_{j=1}^n P_{ij} \left \lvert x_j - \sum_{k=1}^n P_{ik} x_k \right \rvert^2 \le R^2.$$
Collecting terms, we get
\begin{equation*}
    \lvert (D_{X}f)(\epsilon)_i \rvert \le \norm{V}_2 \left ( \norm{A}_2 R^2 \left (\sum_{j=1}^n P_{ij} \lvert \epsilon_j \rvert^2 \right )^{1/2} + \left \lvert \sum_{j=1}^n P_{ij} \epsilon_j \right \rvert + \norm{A}_2 R^2 \lvert \epsilon_i  \rvert \right ).
\end{equation*}
Then, using the inequality $(a + b + c)^2 \le 3(a^2 + b^2 + c^2)$ for any $a,b,c\in \R$, we obtain
\begin{align*}
    \sum_{i=1}^n \lvert (D_{X}f)(\epsilon)_i \rvert^2 &\le 3 \norm{V}_2^2  \left ( \norm{A}_2^2 R^4 (n + 1) + n \right ),
\end{align*}
where we have used that with the triangle inequality and Cauchy-Schwarz inequality
$$ \left \lvert \sum_{j=1}^n P_{ij} \epsilon_j \right \rvert^2 \le \left ( \sum_{j=1}^n P_{ij} \lvert \epsilon_j \rvert \right )^2 \le \sum_{j=1}^n P_{ij}^2\norm{\epsilon}_F^2 \le 1$$
as $\norm{\epsilon}_F^2 = 1$ and $\sum_{j=1}^n P_{ij}^2 \le \sum_{j=1}^n P_{ij} = 1$, and that
$$ \sum_{j=1}^n P_{ij} \lvert \epsilon_j \rvert^2 \le \sum_{j=1}^n \lvert \epsilon_j \rvert^2 = 1.$$

The proof above allows us to recover the following tighter bound but with maybe less natural assumptions on the tokens.

\begin{theorem}
    \label{thm:unnorm_general_bound_tighter}
    Let $Q, K, V \in \R^{k\times d}$ and $A \coloneqq K^\top Q / \sqrt{k}$.
    Denote $f$ unmasked self-attention with parameters $(A, V)$.
    Let $X\in (\R^d)^n$ and $\rho, r > 0$ such that 
    \begin{enumerate}[label=(\roman*), topsep=0pt]
        \item $\max_{1\le i, j \le n}\lvert x_i - x_j\rvert \le r$,
        \item $\max_{1\le i\le n} \lvert Ax_i\rvert \le \rho$.
    \end{enumerate}
    Then, the local Lipschitz constant of $f$ at $X$ is bounded by
    \begin{equation*}
        \norm{D_X f}_2 \le \sqrt{3} \norm{V}_2 \left ( (\rho^2 r^2 + 1) n + \rho^2 r^2 \right )^{1/2}\!\!.
    \end{equation*}
\end{theorem}

Moreover, in practical Transformers architectures, inputs $X$ of self-attention can be written as $\mathrm{norm}(\tilde X)$ for some $\tilde X \in (\R^d)^n$, where $\mathrm{norm}$ stands for LayerNorm of RMSNorm (see Subsection \ref{subsec:normalization}).
In both cases, all inputs are on an ellipsis whose shape depends on the parameters of LayerNorm or RMSNorm, so that there is a choice of parameters $(r, \rho)$ such that for all $\tilde X\in (\R^d)^n$, the input $X = \mathrm{norm}(\tilde X)$ satisfies assumptions $(i)$ and $(ii)$ in Theorem \ref{thm:unnorm_general_bound_tighter}.

\subsection{Weighted Self-Attention}
\label{appsubsec:weighted_sa}

In the whole subsection, $\Sigma_n \coloneqq \{ a\in [0,1]^n : \sum_{i=1}^n a_i = 1\}$ is the simplex.

In view of proving Propositions \ref{prop:lower_bound_sqrt} and \ref{prop:unnorm_lower_bound_mf}, let us introduce a framework for self-attention that extends the Euclidean framework, where the local Lipschitz constant is given by the operator norm of the differential, to probability measures with a finite number of diracs.
We call this framework weighted self-attention.

\begin{definition}[Weighted self-attention]
    \label{appdef:weighted_attention}
    For any vector $a\in \Sigma_n$, denote 
    $\mathcal{P}_a(\R^d) \coloneqq \{ \sum_{i=1}^n a_i \delta_{x_i} : x_1,\dots,x_n\in \R^d\}.$
    We define the Euclidean version of the restriction of self-attention with parameters $(A, V)$ to $\mathcal{P}_a(\R^d)$ in the following way:
    $$f_a\colon (x_1, \dots, x_n) \in (\R^d)^n \mapsto \left ( V \sum_{j=1}^n P_{ij}^a x_j \right )_{1\le i\le n}\in (\R^k)^n$$
    with
    $P_{i}^a \coloneqq \softmax_a\left ( (x_i^\top A^\top x_j)_{1\le j\le n} \right ),$
    where
    $\softmax_a (w_1,\dots, w_n)\coloneqq \left (\frac{a_i e^{w_i}}{\sum_{j}a_j e^{w_j}}\right )_{1\le i\le n}.$
    The function $f_a$ is called weighted self-attention associated to the coefficients $a$.
\end{definition}

Definition \ref{appdef:weighted_attention} is designed so that for any sequence $X\in (\R^d)^n$, it holds
$$F(m_a(X)) = m_a(f_a(X)),$$
with
$m_a(X)\coloneqq \sum_{i=1}^n a_i \delta_{X_{i}}.$
Weighted self-attention provides a representation that is very convenient to study both theoretically and numerically the local Lipschitz constant of self-attention at measures that have a finite number of Diracs but with weights such that they require a massive sequence length to be approximated well by an empirical measure.
This is for example the case of measures of the form $e^{-2 R^2 } \delta_R + (1 - e^{-2 R^2 }) \delta_{-R}$, which are at stake in the proof of Proposition \ref{prop:unnorm_lower_bound_mf}.
To study the Lipschitz continuity of $f_a$, we equip the space $(\R^d)^n$ with the norm
$$\norm{X}_a \coloneqq \left ( \sum_{i=1}^n a_i \lvert x_i \rvert^2 \right )^{1/2},$$
so that we have the following connection with the local Lipschitz constant of mean-field self-attention in the Wasserstein 2 sense.

\begin{lemma}
    \label{applem:loc_lip_jac_weighted}
    Let $X\in (\R^d)^n$ be an input sequence, and $a\in \Sigma_n$ a vector of coefficients.
    Then, we have
    $$\lip_{m_a(X)}^{W_2}F_{\lvert \mathcal{P}_a(\R^d)} = \norm{D_X^af_a}_{2,a},$$
    where $D^a$ is the Jacobian in the space $((\R^d)^n,\norm{\cdot}_a)$ and $\norm{\cdot}_{2,a}$ is the corresponding operator norm.\footnote{As we are in finite dimension, all norms are equivalent so the linear operator $D^a$ is just the usual notion of differential.}
\end{lemma}

This is a nice property, as it provides an optimized way to compute numerically the local Lipschitz constant of $f_a$, just as for Euclidean self-attention.
Lemma \ref{applem:loc_lip_jac_weighted} can be proven with the same steps as for Lemma \ref{lem:link_euclidian_meanfield}.
Moreover, the differential of $f_a$ has the following expression.

\begin{lemma}
    \label{applem:weighted_jacobian}
    Let $X = (x_1, \dots, x_n) \in (\R^d)^n$.
    For all perturbations $\epsilon \coloneqq (\epsilon_1,\dots, \epsilon_n) \in (\R^d)^n$ and all $i \in \{1, \dots, n\}$, we have
    \begin{equation*}
    (D_{X}^af_a)(\epsilon)_i = V \sum_{j=1}^n P^a_{ij} (x_j - \sum_{k=1}^n P^a_{ik}x_k)x_i^\top A^\top \epsilon_j + V\sum_{j=1}^n P^a_{ij} \epsilon_j + V \sum_{j=1}^n P^a_{ij}(x_j - \sum_{k=1}^n P^a_{ik}x_k) x_j^\top A \epsilon_i,
    \end{equation*}
    with $P^a_{ij} \coloneqq a_j e^{ x_i^\top A^\top x_j} / \sum_{k=1}^n a_k e^{ x_i^\top A^\top x_k}$.
\end{lemma}

\subsection{Proof of Proposition \ref{prop:lower_bound_sqrt}}
\label{appsubsec:lower_bound_sqrt}

Proposition \ref{prop:lower_bound_sqrt} follows from the following Lemma.

\begin{lemma}
    Let $\gamma$ be a real eigenvalue of $A$ and $u\in \R^d$ an associated unit eigenvector.
    \begin{enumerate}[nosep,wide, leftmargin=*]
        \item If $\gamma \ge 0$, denote $X \coloneqq (u, \frac{u}{2}, \dots, \frac{u}{2}) \in (\R^d)^n$.
        Then, for any scaling factor $R >0$, it holds
        $$\norm{D_{R X}f}_{2} \ge  \frac{1}{1 + (n-1)e^{-R^2 \gamma/4}} \sqrt{n - 1}.$$
        \item If $\gamma < 0$, denote $X \coloneqq (u, -u, \dots, -u) \in (\R^d)^n$.
        Then, for any scaling factor $R >0$, it holds
        $$\norm{D_{R X}f}_{2} \ge  \frac{1 }{1 + (n-1)e^{- 2 R^2 \left \lvert \gamma \right \rvert}} \sqrt{n - 1}.$$ 
    \end{enumerate}
\end{lemma}

\begin{proof}
    Let us start with the case $\gamma < 0$.
    Let $\tilde X\coloneqq (u, -u)$.
    Using weighted self-attention $f_a$ introduced in Subsection \ref{appsubsec:weighted_sa}, and $a\coloneqq (1/n, 1 - 1/n)$, we have for any $\varepsilon \in \R^d$ and $1\le i\le n$ that
    $$D_{R\tilde X}f_a(\varepsilon)_i = \sum_{j=1}^n P_{ij}^a \varepsilon_j + \sum_{j=1}^n P_{ij}^a x_j (x_j - \sum_{k=1}^n P_{ik}^a x_k)^\top A \varepsilon_i + \sum_{j=1}^n P_{ij}^a (x_j - \sum_{k=1}^n P_{ik}^a x_k)x_i^\top A^\top \varepsilon_j.$$
    Setting $\varepsilon_1 \coloneqq -u$ and $\varepsilon_2 \coloneqq 0$, we get
    $$D_{R\tilde X}f_a(\varepsilon)_2 = P_{21}^a(1 + 2 P_{22}^aR^2\lvert \gamma\rvert) \ge P_{21}^a.$$
    Moreover
    $$P_{21}^a = \frac{1}{1 + (n-1)e^{-2R^2\lvert \gamma \rvert}},$$
    so that
    $$\norm{D_{RX}f}_2 \ge \norm{D_{R\tilde X}f_a}_{2, a}\ge P_{21}^a \sqrt{n-1},$$
    which proves the result.
    When $\gamma \ge 0$, the same method applies with $\varepsilon_1 = u$ and $\varepsilon_2 = 0$.
    Denoting $\tilde X \coloneqq (u, u/2)$ and $a\coloneqq (1/n, 1 - 1/n)$ one obtains
    $$\norm{D_{RX}f}_2 \ge \norm{D_{R\tilde X} f_a }_2\ge P_{21}^a \sqrt{n-1}$$
    and
    $$P_{21}^a = \frac{1}{1 + (n-1)e^{-R^2\gamma / 4}},$$
    which proves the result.
\end{proof}

\subsection{About the scaling factor in Proposition \ref{prop:lower_bound_sqrt}}
\label{appsubsec:scaling_factor}

Let us investigate numerically the scaling factor $2 R^2 \gamma$ that appears in Proposition \ref{prop:lower_bound_sqrt}.
With the model BERT pretrained, and with a batch of five text extracts from the dataset Alice in Wonderland, we plot $\gamma^{(h)} R^2$ for each layer $0\le \ell \le 11$ and for $h\in \{0, 5, 10\}$, where $\gamma^{(h)}$ is the parameter $\gamma$ defined in Proposition \ref{prop:lower_bound_sqrt} associated to the $A^{(h)}$ the weight matrix of head $h$, and $R$ is the mean magnitude of tokens as they enter the attention block of layer $\ell$, defined as $R \colon \sqrt{1/n\sum_{i=1}^n \lvert x_i\rvert^2}$.
We obtain Figure \ref{fig:scaling_factor}.

\begin{figure}
\centering
\includegraphics[]{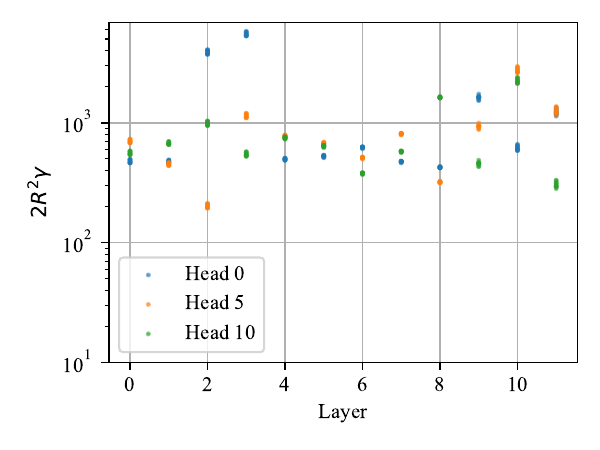}
\caption{Plot of the scaling factor $2R^2 \gamma$ across layers of BERT pretrained for three different heads and 5 text extracts of Alice in Wonderland (50 tokens for each extract).}
\label{fig:scaling_factor}
\end{figure}

\subsection{Proof of Proposition \ref{prop:unnorm_lower_bound_mf}}
\label{appsubsec:unnorm_lower_bound_mf}

Let us prove the following result, which implies Proposition \ref{prop:unnorm_lower_bound_mf}.

\begin{proposition}
    Let $R>0$. 
    Assume that $k = d$ and $V = I_d$, and denote $\gamma_1\ge \dots \ge \gamma_\nu$ the real eigenvalues of $A$.
    Then, the following claims hold.
    \begin{enumerate}
        \item If $\gamma_1 \ge -8\gamma_\nu$, denoting $C \coloneqq \frac{ \gamma_1}{8}$, there exists a function $\theta \colon [0, +\infty) \to [0, +\infty)$ such that $\theta(R) \to_{R\to +\infty} 1$ and:
        $$\lip^{W_2}(F_{\lvert \pcal(B_R)}) \ge  \theta (R) \frac{C}{2} R^2  e^{C R^2}.$$
        Moreover, the right-hand side is equivalent to the Lipschitz constant of mean-field self-attention at the probability measure $e^{-2CR^2} \delta_{Ru_1} + (1 - e^{-2CR^2}) \delta_{(R/2)u_1}$. 
        \item If $\gamma_1 < -8\gamma_\nu$, denoting $C' \coloneqq \lvert \gamma_\nu \rvert$, there exists a function $\theta \colon [0, +\infty) \to [0, +\infty)$ such that $\theta(R) \to_{R\to +\infty} 1$ and:
        $$\lip^{W_2}(F_{\lvert \pcal(B_R)}) \ge  \theta(R) \frac{C'}{2} R^2 e^{C' R^2 }.$$
        Moreover, the right-hand side is equivalent to the Lipschitz constant of mean-field self-attention at the probability measure $e^{-2C'R^2} \delta_{Ru_\nu} + (1 - e^{-2C'R^2}) \delta_{-Ru_\nu}$. 
    \end{enumerate}
\end{proposition}

\begin{proof}
    Let us detail the proof in the second case only, as the first case is very similar.
    According to Lemma \ref{applem:weighted_jacobian} with $V = I_d$, for any $X = (x_1, x_2) \in (\R^d)^2$ and $a = (a_1, a_2)\in \Sigma_2$, for any $\epsilon_2 \in \R^d$, denoting $\epsilon \coloneqq (0, \epsilon_2) \in (\R^d)^2$ we have
    $$D_X^a f_a (\epsilon)_1 = (P_{12}^a I_d + 2 P_{11}^a P_{12}^a x_2 x_1^\top A^\top) \epsilon_2.$$
    Let $u_1, \dots, u_\nu$ a family of unit eigenvectors of $A$, associated to the eigenvalues $\gamma_1 \ge \dots \ge \gamma_\nu$.
    Let $R > 0$.
    Set $a_2 \coloneqq e^{-2\lvert \gamma_\nu \rvert R^2}$, so that $a_1 = 1 - e^{-2\lvert \gamma_\nu \rvert R^2}$, and choose $x_1 \coloneqq Ru_\nu$ and $x_2\coloneqq -R u_\nu$.
    Then
    \begin{equation} 
        \label{appeq:exploding_term}
        \norm{D_X^a f_a(\epsilon)}_a \ge \sqrt{a_1} \lvert P_{12}^a (I_d + 2 P_{11}^a R^2 \lvert\gamma_\nu \rvert u_\nu u_\nu^\top )\epsilon_2\rvert.
    \end{equation}
    Now let us notice that
    $$P_{12}^a = \frac{a_2 e^{\lvert \gamma_\nu \rvert R^2}}{a_1 e^{- \lvert \gamma_\nu \rvert R^2} + a_2 e^{ \lvert \gamma_\nu \rvert R^2}} = \frac{e^{- \lvert \gamma_\nu \rvert R^2}}{(1 - e^{-2\lvert \gamma_\nu\rvert R^2})e^{- \lvert \gamma_\nu \rvert R^2} + e^{- \lvert \gamma_\nu \rvert R^2}}\to_{R\to +\infty} \frac{1}{2},$$
    and
    $$P_{11}^a = 1 - P_{12}^a \to_{R\to +\infty} \frac{1}{2}$$
    as well.
    Going back to Equation (\ref{appeq:exploding_term}) and setting $\epsilon_2 \coloneqq u_\nu$, we get
    $$\norm{D_X^a f_a(\epsilon)}_a \gtrsim \frac{1}{2} (1 + R^2 \lvert\gamma_\nu \rvert) \lvert u_\nu\rvert \gtrsim \frac{1}{2}R^2 \lvert \gamma_\nu \rvert,$$
    where $f(R) \gtrsim g(R)$ means that there exists a function $\theta \colon \R \to \R_+$ such that $\theta(R) \to_{R\to +\infty} 1$ and $f(R) \ge \theta(R) g(R)$.
    Finally, we get
    $$\frac{\norm{D_X^a f_a(\epsilon)}_a}{\norm{\epsilon}_a} \gtrsim \frac{1}{2} \lvert \gamma_\nu \rvert R^2{a_1}^{-1/2} = \frac{1}{2} \lvert \gamma_\nu \rvert R^2 e^{\lvert \gamma_\nu\rvert R^2},$$
    which proves the result as $\norm{D_X^a f_a}_{2,a} = \sup_{\epsilon \in (\R^d)^2\setminus \{ 0\}} \frac{\norm{D_X^a f_a(\epsilon)}_a}{\norm{\epsilon}_a}$.
\end{proof}

\subsection{An Example of Quadratic Growth with the Radius}
\label{appsubsec:quadratic_growth}

For simplicity, assume that $A=I_d$ and $V=I_d$.
Let $X = (0, x_2 \dots, x_n) \in B_R^n$.
\citet{kim2021lipschitz} show that, as all norms are equivalent in finite dimension, the local Lipschitz constant of $f$ at $X$ grows like the empirical variance of $X$.
Taking $x_2 = \dots = x_j = Re_1$ and $x_{j+1} = \dots = x_n = -Re_1$ with $e_1 = (1, 0,\dots, 0) \in \R^d$ and $j$ such that $\lvert 2j - 1 - n\rvert \le 1$ and denoting $X_R$ the resulting configuration, we obtain that the empirical variance of $X_R$ grows like $R^2$ up to a constant factor that is close to 1.

This behavior can be easily checked numerically (see Figure \ref{fig:quadratic_growth}).

\begin{figure}
\centering
\includegraphics[]{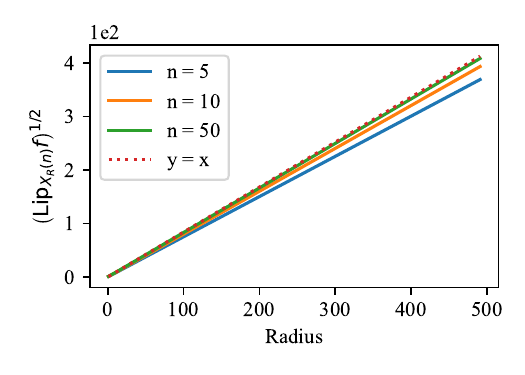}
\caption{Linear growth of the square root of the Lipschitz constant of self-attention in the configuration $X_R(n)$.}
\label{fig:quadratic_growth}
\end{figure}

\section{Proofs of \autoref{sec:masked}}

We have the following formula for the Jacobian of masked self-attention.

\begin{lemma}
    \label{applem:jacobian_formula_masked}
    Let $X = (x_1, \dots, x_n) \in (\R^d)^n$.
    For all perturbations $\epsilon \coloneqq (\epsilon_1,\dots, \epsilon_n) \in (\R^d)^n$ and all $i \in \{1, \dots, n\}$, we have
    \begin{equation*}
    (D_{X}f^m)(\epsilon)_i = V \sum_{j=1}^i P_{ij} (x_j - \sum_{k=1}^i P_{ik}x_k)x_i^\top A^\top \epsilon_j + V\sum_{j=1}^i P_{ij} \epsilon_j + V \sum_{j=1}^i P_{ij}(x_j - \sum_{k=1}^i P_{ik}x_k) x_j^\top A \epsilon_i,
    \end{equation*}
    with $P_{ij} \coloneqq e^{ x_i^\top A^\top x_j} / \sum_{k=1}^i e^{ x_i^\top A^\top x_k}$.
\end{lemma}

\subsection{Proofs of \autoref{thm:masked_general_bound} and \autoref{thm:large_R_masked}}
\label{appsubsec:masked_general_bound}

In view of \autoref{applem:jacobian_formula_masked}, following the same steps as in the proof of \autoref{thm:unnorm_general_bound} leads to \autoref{thm:masked_general_bound}.
Likewise, the proof of \autoref{thm:large_R_masked} is the same as for \autoref{thm:unnorm_large_R_regime}.

\subsection{Proof of \autoref{thm:masked_mean_field}}
\label{appsubsec:masked_mean_field_bound}

Let $\bar \mu$ and $\bar \nu$ be two distinct measures in $\pcal([0,1]\times B_R)$.
Assume that $\bar \mu$ and $\bar \nu$ have the same first marginal, denoted $\theta$ (otherwise, we consider that they are associated with an infinite Lipschitz ratio).
We have
\begin{align*}
    d_p(F^m(\bar \mu), F^m(\bar \nu)) &\le d_p((\Gamma_{\bar \mu})_\sharp \bar \mu, (\Gamma_{\bar \mu})_\sharp \bar \nu) + d_p((\Gamma_{\bar \mu})_\sharp \bar \nu, (\Gamma_{\bar \nu})_\sharp \bar \nu) \\
    &\le \left ( \int_0^1 W_p\left ((\Gamma_{\bar \mu}^s)_\sharp \mu^s, (\Gamma_{\bar \mu}^s)_\sharp \nu^s\right )^p \right )^{1/p} \\
    &\phantom{space} + \left ( \int_0^1 W_p\left (\Gamma_{\bar \mu}^s)_\sharp \nu^s, (\Gamma_{\bar \nu}^s)_\sharp \nu^s \right ) \right )^{1/p},
\end{align*}
where we denote
$$\Gamma_{\bar \mu}^s(x) := \frac{\int Vy G(x,y)\mathbf{1}_{\tau \le s} \dd \bar \mu(\tau, y)}{\int G(x,y)\mathbf{1}_{\tau \le s} \dd \bar \mu(\tau, y)}.$$
Using Lemma \ref{applem:lip_wass1}, we get
$$W_p\left ((\Gamma_{\bar \mu}^s)_\sharp \mu^s, (\Gamma_{\bar \mu}^s)_\sharp \nu^s\right ) \le \lip(\Gamma_{\bar \mu}^s) W_p(\mu^s, \nu^s)$$
where the Lipschitz constant is taken on $B_R$.
A similar computation as for traditional self-attention shows that for all $0\le s\le 1$ and $x\in B_R$ we have
$$\norm{D_x \Gamma^s}_2 \le \norm{V}_2 \norm{A}_2 R^2,$$
so that
$$ \lip(\Gamma_{\bar \mu}^s) \le\norm{V}_2 \norm{A}_2 R^2.$$
To bound the second term in the previous inequality, we use Lemma \ref{applem:lip_wass2}, to get
$$ W_p\left ( (\Gamma_{\bar \mu}^s )_\sharp \nu^s, (\Gamma_{\bar \nu}^s)_\sharp \nu^s \right ) \le \norm{\Gamma_{\bar \mu}^s - \Gamma_{\bar \nu}^s}_{L^\infty(B_R)}.$$
Again, a similar computation as for traditional self-attention shows that
$$\norm{\Gamma_{\bar \mu}^s - \Gamma_{\bar \nu}^s}_{L^\infty (B_R)} \le \norm{V}_2 (2\norm{A}_2 R^2) e^{2\norm{A}_2 R^2} W_p(\mu^s, \nu^s),$$
which concludes the proof.

\section{Experiments}

\subsection{Power Iteration}
\label{appsubsec:power_iteration}

Let $X\in (\R^d)^n$.
To compute numerically $\norm{D_Xf}_2$, we pick an initialisation
\begin{gather*}
    u_0 \sim \otimes^{n\times d} \mathcal{N}(0, 1) \in \R^{n\times d} \\
    u_0 \leftarrow{u_0 / \norm{u_0}_F}
\end{gather*}
and then repeat the following steps until convergence:
\begin{gather*}
    v_k = (D_X f)^\top (D_Xf) u_k \\
    \mu_k = v_k^\top u_k \\
    u_{k+1} = \frac{v_k}{\norm{v_{k+1}}_F}
\end{gather*}
where $v_k$ is computed by doing successively a Jacobian-vector product and a vector-Jacobian product.
It is well known \cite{sra2012optimization} that with this method, $\mu_k$ converges to $\norm{D_Xf}_2$.

\subsection{GPT-2}
\label{appsubsec:gpt}

\begin{figure}
\centering
\includegraphics[width=0.7\textwidth]{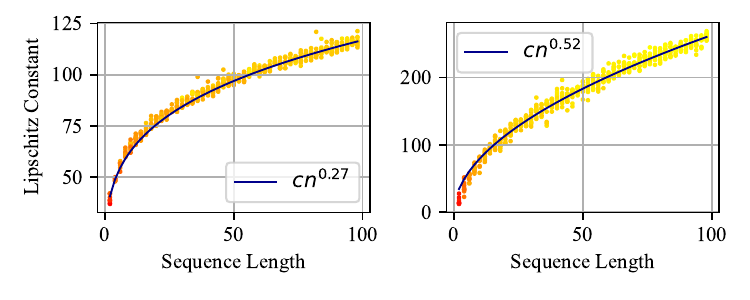}
\caption{Scatter plots of the local Lipschitz constant of masked self-attention for GPT-2 pretrained as a function of the sequence length, on the dataset Alice in Wonderland.
The first column corresponds to masked self-attention layer 0, and the second column to layer 6.}
\label{fig:gpt_trained}
\end{figure}

We do the same experiments as in Section \ref{sec:experiments} on GPT-2 pretrained \cite{radford2019language}.
We see in Figure \ref{fig:gpt_trained} that the behavior is different than what we observe in Figure \ref{fig:real_data}.
This is explained by the fact that the magnitude of tokens depends on the sequence length for pretrained GPT-2, due to the learned positional encoding (see Figures \ref{fig:wpe_norms} and \ref{fig:correlation_n_R}).
Therefore, the bound of Theorem \ref{thm:unnorm_general_bound} still holds but the growth rate can be faster than $\sqrt{n}$.

\begin{figure}
\centering
\begin{minipage}{.45\textwidth}
    \centering
    \includegraphics[width=.9\linewidth]{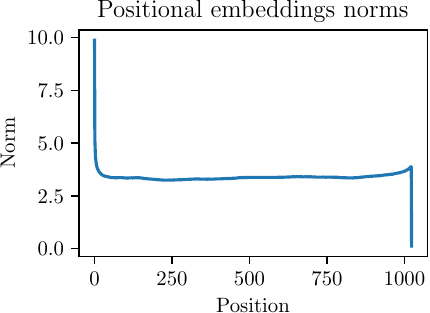}
    \caption{Norm of the positional embeddings of GPT-2 pretrained, ordered by position. The very first tokens are associated to positional embeddings of much larger magnitude, which makes $n$ and $R$ dependent from the very beginning of the architecture.}
    \label{fig:wpe_norms}
\end{minipage}~~~
\begin{minipage}{.45\textwidth}
    \centering
    \includegraphics[width=.9\linewidth]{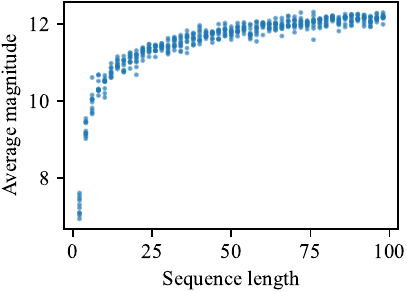}
    \caption{Scatter plot of the average magnitude of tokens $R\coloneqq \sqrt{1/n\sum_{i=1}^n \lvert x_i\rvert^2}$ as a function of the sequence length, for the same dataset as in Figure \ref{fig:gpt_trained}, right column.}
    \label{fig:correlation_n_R}
\end{minipage}
\end{figure}

\subsection{Self-attention with biases}
\label{appsubsec:sa_biases}

Some Transformer architectures such as BERT \cite{devlin2018bert} add biases $(b_Q, b_K, b_V) \in \R^k$ in the formula of self-attention:
\begin{equation}
    \label{eq:sa_biases}
    f^b\colon (x_1, \dots, x_n) \mapsto \Big ( V {\textstyle \sum_{j=1}^n} P_{ij} x_j + b_V \Big )_{1 \le i \le n} \in (\R^k)^n,
\end{equation}
\begin{equation*}
    \text{with}\quad P_i \coloneqq \softmax \left ( (Qx_i+b_Q)^\top (K x_j + b_K))_{1\le j\le n} \right ),
\end{equation*}
where we absorbed the factor $1/\sqrt{k}$ in $Q,K, b_Q$ and $b_K$ to alleviate notations.

\begin{remark}
Note that $b_K$ has no influence on the value of $f^b$:
$$\softmax \left ( (Qx_i+b_Q)^\top (K x_j + b_K))_{1\le j\le n} \right ) = \softmax \left ( (Qx_i+b_Q)^\top K x_j)_{1\le j\le n} \right ),$$
as $b_K$ in only involved in terms that are independent of $j$.
\end{remark}

How do biases in self-attention affect the bound in Theorem \ref{thm:unnorm_general_bound}?
We start by computing the Jacobian of self-attention with biases.

\begin{lemma}
    Let $X = (x_1, \dots, x_n) \in (\R^d)^n$.
    Let $f^b$ be a self-attention module with biases, defined as in Equation (\ref{eq:sa_biases}).
    For all perturbations $\epsilon \coloneqq (\epsilon_1,\dots, \epsilon_n) \in (\R^d)^n$ and all $i \in \{1, \dots, n\}$, we have
    \begin{equation*}
    (D_{X}f^b)(\epsilon)_i = V \sum_{j=1}^n P_{ij} (x_j - \sum_{k=1}^n P_{ik}x_k)(Qx_i + b_Q)^\top K \epsilon_j + V\sum_{j=1}^n P_{ij} \epsilon_j + V \sum_{j=1}^n P_{ij}(x_j - \sum_{k=1}^n P_{ik}x_k) x_j^\top A \epsilon_i,
    \end{equation*}
    with $P_{ij} \coloneqq e^{ (Qx_i + b_Q)^\top (K x_j + b_K)} / \sum_{k=1}^n e^{ (Qx_i + b_Q)^\top (K x_k + b_K)}$.
\end{lemma}

The same steps as in the proof of Theorem \ref{thm:unnorm_general_bound} lead to the following bound.

\begin{theorem}
    Let $Q, K, V \in \R^{k\times d}$ and $A \coloneqq K^\top Q / \sqrt{k}$.
    Let $R > 0$ and $n \in \mathbb N$.
    Unmasked self-attention with biases $f^b$ with parameters $(A, V)$, defined in Equation (\ref{eq:sa_biases}), is Lipschitz continuous on the set $B_R^n$, with
    \begin{equation*}
        \lip \Big ( f _{\lvert B_R^n} \Big ) \le \sqrt{3} \norm{V}_2 \left ( \norm{A}_2^2 R^4  + n \left (\norm{K}_2 (\norm{Q}_2 R + \lvert b_Q \rvert)^2 R^2 + 1 \right ) \right )^{1/2}\!\!.
    \end{equation*}
\end{theorem}

\end{document}